\documentclass[twoside]{article}

\usepackage[accepted]{aistats2025}

\usepackage[round]{natbib}
\renewcommand{\bibname}{References}
\renewcommand{\bibsection}{\subsubsection*{\bibname}}

\usepackage{amsmath,amsthm,amssymb,amsfonts,bm}
\usepackage{xspace}
\usepackage{hyperref}
\usepackage{url}

\usepackage{wrapfig}
\usepackage{appendix}
\usepackage{titletoc}

\usepackage[noend]{algorithmic}
\usepackage{algorithm}

\newcommand{\PARAMETERS}{\STATE \hspace{-3.5mm}{\bf Params:}\xspace}

\usepackage{enumitem}
\setitemize[1]{leftmargin=4mm,itemsep=0pt,topsep=-2pt}
\setitemize[2]{leftmargin=4mm,itemsep=0pt,topsep=0pt}

\usepackage{hyperref}
\hypersetup{
	colorlinks=true,
	linkcolor=blue,
	citecolor=blue, %
	filecolor=blue, %
	urlcolor=blue %
}
\usepackage[capitalize,noabbrev,nameinlink]{cleveref}

\newtheorem{theorem}{Theorem}[section]
\newtheorem{lemma}[theorem]{Lemma}
\newtheorem{proposition}[theorem]{Proposition}
\newtheorem{fact}[theorem]{Fact}
\newtheorem{remark}[theorem]{Remark}
\theoremstyle{definition}
\newtheorem{definition}[theorem]{Definition}

\def\ddefloop#1{\ifx\ddefloop#1\else\ddef{#1}\expandafter\ddefloop\fi}
\def\ddef#1{\expandafter\def\csname #1\endcsname{\ensuremath{\mathbb{#1}}}}
\ddefloop ABCDFGHIJKLMNOQRTUVWXYZ\ddefloop  %
\DeclareMathOperator*{\E}{\mathbb{E}}
\def\ddef#1{\expandafter\def\csname c#1\endcsname{\ensuremath{\mathcal{#1}}}}
\ddefloop ABCDEFGHIJKLMNOPQRSTUVWXYZ\ddefloop  %
\def\ddef#1{\expandafter\def\csname b#1\endcsname{\ensuremath{\bm #1}}}
\ddefloop ABCDEFGHIJKLMNOPQRSTUVWXYZabcdeghijklnopqrstuvwxyz\ddefloop  %

\usepackage{xcolor}
\definecolor{Ggreen}{RGB}{60, 186, 84}

\usepackage[disable]{todonotes}

\DeclareMathOperator*{\Ex}{\mathbb{E}}
\newcommand{\eps}{\varepsilon}

\newcommand{\prn}[1]{\left ( #1 \right )}
\newcommand{\sq}[1]{\left [ #1 \right ]}

\newcommand{\Deps}{D_{e^{\eps}}}

\newcommand{\KL}{\mathsf{KL}}

\newcommand{\BnB}{Balls-and-Bins\xspace}

\newcommand{\DP}{\mathsf{DP}}

\newcommand{\SGD}{\mathsf{SGD}}
\newcommand{\DPSGD}{\mathsf{DP}\text{-}\mathsf{SGD}}

\newcommand{\ABLQ}{\mathsf{ABLQ}}

\newcommand{\DPSGDG}{\DPSGD_{\cG}}

\newcommand{\DPSGDP}{\DPSGD_{\cP}}
\newcommand{\DPSGDS}{\DPSGD_{\cS}}
\newcommand{\DPSGDB}{\DPSGD_{\cB}}
\newcommand{\ABLQG}{\ABLQ_{\cG}}
\newcommand{\ABLQD}{\ABLQ_{\cD}}
\newcommand{\ABLQP}{\ABLQ_{\cP}}
\newcommand{\ABLQS}{\ABLQ_{\cS}}
\newcommand{\ABLQB}{\ABLQ_{\cB}}

\newcommand{\epsG}{\eps_{\cG}}

\newcommand{\deltaG}{\delta_{\cG}}
\newcommand{\deltaD}{\delta_{\cD}}
\newcommand{\deltaP}{\delta_{\cP}}
\newcommand{\deltaS}{\delta_{\cS}}
\newcommand{\deltaB}{\delta_{\cB}}

\newcommand{\PD}{P_{\cD}}
\newcommand{\QD}{Q_{\cD}}
\newcommand{\PS}{P_{\cS}}
\newcommand{\QS}{Q_{\cS}}
\newcommand{\PB}{P_{\cB}}
\newcommand{\QB}{Q_{\cB}}

\newcommand{\PP}{P_{\cP}}
\newcommand{\QP}{Q_{\cP}}

\newcommand{\event}{\Gamma}
\newcommand{\dominates}{\succcurlyeq}%

\newcommand{\CDF}{\mathsf{CDF}}
\newcommand{\Beta}{\mathsf{Beta}}
\newcommand{\Unif}{\mathsf{Unif}}
\newcommand{\Ber}{\mathsf{Ber}}

\begin{document}

\runningauthor{Chua, Ghazi, Harrison, Leeman, Kamath, Kumar, Manurangsi, Sinha, Zhang}

\twocolumn[

\aistatstitle{Balls-and-Bins Sampling for DP-SGD}

\aistatsauthor{
	Lynn Chua, Badih Ghazi, Charlie Harrison, Ethan Leeman, \\[.5mm]\bf
	Pritish Kamath, Ravi Kumar, Pasin Manurangsi, Amer Sinha, Chiyuan Zhang\\[-3mm]\mbox{}
}
\aistatsaddress{\large Google}
]

\begin{abstract}
We introduce the {\em Balls-and-Bins} sampling for differentially private (DP) optimization methods such as DP-SGD.
While it has been common practice to use some form of shuffling in DP-SGD implementations, privacy accounting algorithms have typically assumed that Poisson subsampling is used instead.
Recent work by \citet{chua24private}, however, pointed out that shuffling based DP-SGD can have a much larger privacy cost in practical regimes of parameters.
In this work we show that the \BnB sampling achieves the ``best-of-both'' samplers, namely, the implementation of \BnB sampling is similar to that of Shuffling and models trained using DP-SGD with \BnB sampling achieve utility comparable to those trained using DP-SGD with Shuffling at the same noise multiplier, and yet, \BnB sampling enjoys similar-or-better privacy amplification as compared to Poisson subsampling in practical regimes.

\end{abstract}

\section{INTRODUCTION} \label{sec:intro}

Training differentiable models, e.g.,  neural networks, with noisy gradients via first-order methods such as stochastic gradient descent (SGD), has become a common approach for making the training pipelines satisfy differential privacy. Since its introduction by \citet{abadi16deep}, this approach of $\DPSGD$ has been the basis of open source implementations in \cite{tf_privacy}, PyTorch Opacus~\citep{yousefpour21opacus} and JAX Privacy~\citep{jax-privacy2022github}. $\DPSGD$ has been widely applied across various domains~\citep[e.g.,][]
{de22unlocking,dockhorn2022differentially,anil22dpbert,he2022exploring,igamberdiev2023dp,tang2024private}.%

$\DPSGD$ processes the training data in a sequence of steps, where at each step, a noisy estimate of the average gradient over a mini-batch is computed and used to perform a first-order update over the differentiable model; a formal description is provided in \Cref{alg:dpsgd}.  In summary, the noisy (average) gradient is obtained by {\em clipping} the gradient $g$ for each example in the mini-batch to have norm at most $C$ (a preset bound), namely $[g]_C := g \cdot \min\{1, C / \|g\|_2\}$, computing the sum over the batch, and then adding independent zero-mean noise drawn from the Gaussian distribution of scale $\sigma C$ to each coordinate of the gradient sum.
This could then be scaled by the mini-batch size\footnote{When the mini-batch size is a random variable, the scaling has to be done with a fixed value, e.g., the expected mini-batch size, and not the realized mini-batch size.} to obtain a noisy average gradient.
The privacy guarantee of the mechanism depends on the following: the noise scale $\sigma$, the number of examples in the training dataset, the size of mini-batches, the number of training steps, and the {\em mini-batch generation process}.

Most deep learning systems in recent years process mini-batches of fixed-size by sequentially iterating over the dataset, perhaps after applying a global or some other form of {\em shuffling} to the dataset.
The privacy analysis for such a mechanism however has been technically challenging due to the correlated nature of the mini-batches. To simplify the privacy analysis, \cite{abadi16deep} considered {\em Poisson subsampling}, wherein each mini-batch is sampled independently by including each example independently with a fixed probability. However, Poisson subsampling is rarely implemented and instead it has become common practice to use some form of shuffling in model training, but to report privacy parameters as if Poisson subsampling was used \citep[\S4.3]{ponomareva23dpfy}; a notable exception is the PyTorch Opacus library~\citep{yousefpour21opacus} that supports Poisson subsampling for $\DPSGD$ and provides privacy accounting methods for it.

{\bf Adaptive Batch Linear Queries~\citep{chua24private}.}
For any batch generation process, the privacy analysis of $\DPSGD$, particularly in the case of non-convex models such as deep neural networks, is typically done by viewing it as a post-processing of an {\em Adaptive Batch Linear Queries ($\ABLQ$)} mechanism (\Cref{alg:ablq}) that releases estimates of a sequence of adaptively chosen linear queries as produced by an {\em adaptive query method $\cA$}, on the mini-batches obtained using a {\em batch generator} $\cG$.

Given a dataset of $n$ examples, the batch generator $\cG$ can be any algorithm that generates a sequence $S_1, \ldots, S_T \subseteq [n]$ of mini-batches. We use $\cG_{b,T}$ to emphasize the number of batches $T$ and the (expected) batch size $b$, but often omit the subscript when it is clear from context.
$\ABLQG$ processes the batches generated by $\cG$ in a sequential order to produce a sequence $(g_1, \ldots, g_T)$.  Here, each $g_t \in \R^d$ is the average value of $\psi_t(x)$ over the batch $S_t$ with added zero-mean Gaussian noise of scale $\sigma$ to all coordinates, where the query $\psi_t : \cX \to \B^d$ (for $\B^d := \{ v \in \R^d : \|v\|_2 \le 1\}$) is produced by the adaptive query method $\cA$, based on the previous responses $g_1, \ldots, g_{t-1}$.
$\DPSGD$ with batch generator $\cG$ (denoted $\DPSGDG$) can be obtained as a post-processing of the output of $\ABLQG$, with the adaptive query method that maps examples to the clipped gradient at the current iterate, namely $\psi_t(x) := [\nabla_{\bw} \ell(\bw_{t-1}, x)]_1$ (the clipping norm $C$ is considered to be $1$ w.l.o.g.), and where only the last iterate $\bw_T$ is revealed.

We consider the following batch generators (see formal description in \Cref{app:batch-gen}):
\begin{itemize}[topsep=-2pt,itemsep=-2pt]
    \item Deterministic $\cD_{b, T}$:  generates $T$ batches each of size $b$ in the given sequential order of the dataset,
    \item Shuffle $\cS_{b,T}$: similar to $\cD_{b, T}$, but first applies a random permutation to the dataset, and
    \item Poisson $\cP_{b,T}$: each batch independently includes each example with probability $\frac bn$.
\end{itemize}
We drop the subscripts of each generator whenever it is clear from context.
For any batch generator $\cG$, e.g.,  $\cD, \cP, \cS$, we use $\deltaG(\eps)$ to denote the {\em privacy loss curve} of $\ABLQG$. Namely, for all $\eps > 0$, let $\deltaG(\eps)$ be the smallest $\delta \ge 0$ such that $\ABLQG$ satisfies $(\eps, \delta)$-$\DP$ for {\em all} choices of the underlying adaptive query method $\cA$, and $\epsG(\delta)$ is defined similarly.

It might appear that the privacy analysis of $\ABLQG$ would be worse than that of $\DPSGDG$ in general, since $\ABLQG$ releases estimates to {\em all} intermediate linear queries, whereas, $\DPSGDG$ only releases the {\em final} iterate. However, a recent interesting work of \cite{annamalai2024loss} shows that for general non-convex losses, the privacy analysis of the {\em last-iterate} of $\DPSGDP$ is no better than that of $\ABLQP$. This suggests that at least without any further assumptions on the loss functions, it might not be possible to improve the privacy analysis of $\DPSGDG$ beyond that provided by $\ABLQG$.

\cite{chua24private} showed that the privacy guarantee of $\ABLQS$ can be significantly worse than that of $\ABLQP$ in practical regimes of privacy parameters, especially for small noise scale $\sigma$; subsequently such gaps were also observed in empirical privacy by \cite{annamalai24shuffle}.
This seriously challenges the common practice of using shuffling in $\DPSGD$  while claiming privacy guarantees based on Poisson subsampling.
Concurrently, \cite{lebeda2024avoiding} consider the batch generator $\cW$ where each batch is generated  independently by sampling examples from the dataset without replacement, until the desired batch size is met, and show that $\ABLQ_{\cW}$ also exhibits worse privacy guarantees compared to $\ABLQP$.
The central question we consider is: 

{\sl
Is there a batch generator that is similar to Shuffling in terms of implementation simplicity and model utility, but has favorable privacy analysis, namely similar to or better than Poisson subsampling?
}

{
\begin{algorithm}[t]
\caption{$\DPSGD_\cG$~\citep{abadi16deep}}
\label{alg:dpsgd}
\begin{algorithmic}
\PARAMETERS
Differentiable loss $\ell : \R^d \times \cX \to \R^d$,
\STATE\phantom{{\bf Param}} Batch generator $\cG_{b,T}$,
initial state $w_0$,
\STATE\phantom{{\bf Param}} clipping norm $C$, noise scale $\sigma$.
\REQUIRE Dataset $\bx = (x_1, \ldots, x_n)$.
\ENSURE Final model state $\bw_T \in \R^d$.
\STATE $(S_1, \ldots, S_T) \gets \cG(n)$ \hfill 
\FOR{$t = 1, \ldots, T$}
\STATE $g_t \gets \frac{1}{b} \prn{\sum_{x \in S_t} [\nabla_{\bw} \ell(\bw; x)]_C + \cN(0, \sigma^2 C^2 I_d)}$
\STATE $\bw_t \gets \bw_{t-1} - \eta_t g_t$ \hfill \textcolor{black!50!Ggreen}{$\triangleright$ or any other optimizer step}
\ENDFOR
\RETURN $\bw_T$
\end{algorithmic}
\end{algorithm}
}

{
\begin{algorithm}[t]
\caption{$\ABLQG$: Adaptive Batch Linear Queries (as formalized in \cite{chua24private})}
\label{alg:ablq}
\begin{algorithmic}
\PARAMETERS Batch generator $\cG$, noise scale $\sigma$, and (adaptive) query method $\cA: (\R^d)^* \times \cX \to \B^d$.
\REQUIRE Dataset $\bx = (x_1, \ldots, x_n)$.
\ENSURE Query estimates $g_1, \ldots, g_T \in \R^d$
\STATE $(S_1, \ldots, S_T) \gets \cG(n)$ \hfill
\FOR{$t = 1, \ldots, T$}
    \STATE $\psi_t(\cdot) := \cA(g_1, \ldots, g_{t-1}; \cdot)$
    \STATE $g_t \gets \sum_{i \in S_t} \psi_t(x_i) + e_t$ for $e_t \sim \cN(0, \sigma^2 I_d)$
\ENDFOR
\RETURN $(g_1, \ldots, g_T)$
\end{algorithmic}
\end{algorithm}
}

\begin{algorithm}[t]
\caption{$\cB_{T}$: \BnB Sampler}
\label{alg:bnb-batch}
\begin{algorithmic}
\PARAMETERS Number of batches $T$.
\REQUIRE Number of datapoints $n$.
\ENSURE Seq. of batches $S_1, \ldots, S_T \subseteq [n]$.
\FOR{$t = 1, \ldots, T$}
    \STATE $S_{t} \gets \emptyset$
\ENDFOR
\FOR{$i = 1, \ldots, n$}
    \STATE $S_{t} \gets S_t \cup \{i\}$ for uniformly random $t \in [T]$
\ENDFOR
\RETURN $S_1, \ldots, S_T$
\end{algorithmic}
\end{algorithm}

{\bf Our Contributions.}
Towards answering this question, we introduce the {\em \BnB} generator $\cB_{T}$~(\Cref{alg:bnb-batch}), which operates by placing each example in a random batch. This sampler exhibits behavior similar to Shuffle, with each example appearing in precisely one batch, and is also similar to implement. On the other hand, the marginal distribution over each batch is exactly the same as that of Poisson subsampling. This overcomes the privacy lower bound of \cite{chua24private, lebeda2024avoiding} by making the positions of the different examples to be independent, thereby preventing non-differing examples from leaking information about the presence or absence of the differing example in any given batch. 

We identify a tightly dominating pair~(\Cref{def:dominating-pair}) for $\ABLQB$, thereby allowing a tight privacy analysis. We show that $\ABLQB$ enjoys better privacy guarantees compared to $\ABLQD$ and $\ABLQS$ in all regime of parameters. This is in sharp contrast to $\ABLQP$, which can have worse privacy guarantees than even $\ABLQD$ at large $\eps$~\citep{chua24private}.

We use a Monte Carlo method for estimating  $\deltaB(\eps)$. However, naive Monte Carlo methods are inefficient when estimating small values of $\delta$, or when the number of steps $T$ is large.
Our key contributions here are to develop the techniques of {\em importance sampling}, to handle small $\delta$ values, and {\em order statistics sampling}, a new technique to handle a large number of steps. We believe the latter is of independent interest beyond DP.

Finally, we evaluate $\DPSGD$ on some practical datasets and observe that the model utility of $\DPSGDB$ is comparable to that of $\DPSGDS$
at the same noise scale $\sigma$.
On the other hand, for each setting of parameters used, we use our Monte Carlo estimation method to show that the privacy guarantees of $\ABLQB$ are similar/better relative to $\ABLQP$, whereas, the privacy guarantees of $\ABLQS$ are much worse.

{\bf Related Work.}
\cite{balle20checkin} considered a model of {\em random check-ins} in the context of distributed (a.k.a. federated) learning, which is essentially the same as balls-and-bins sampling. Their analysis however relies on the amplification properties of shuffling and does not lead to better privacy guarantees than those known for shuffling.

An independent and concurrent work of \cite{choquettechoo24near} also considered the Balls-and-Bins sampling method, although in the context of the so-called DP-FTRL algorithm, the variant of DP-SGD that adds correlated noise at each step~\citep{kairouz21practical,mcmahan22dpmf}. They also use a Monte Carlo method to compute the privacy parameters.
Since DP-SGD is a special case of DP-FTRL (with independent noise), the dominating pair we identify is a special case of the dominating pair for DP-FTRL, which depends on the specific correlation matrix. \cite{choquettechoo24near} mention ``a more careful sampler and concentration analysis'' as an avenue for improving the sample complexity of the Monte Carlo estimation. In our work, we develop techniques based on importance sampling and order statistics sampling to improve the cost of the Monte Carlo method.  Extending these techniques to the  general DP-FTRL setting is an interesting future direction.

\section{PRELIMINARIES}\label{sec:prelims}
\vspace{-2mm}

A {\em mechanism} $\cM$ maps input datasets to distributions over an output space.  Namely, for $\cM : \cX^* \to \Delta_{\cO}$, on input {\em dataset} $\bx = (x_1, \ldots, x_n)$ where each {\em record} $x_i \in \cX$, $\cM(\bx) \in \Delta_{\cO}$ is a probability distribution over the output space $\cO$; we often use $\cM(\bx)$ to denote the underlying random variable as well.
Two datasets $\bx$ and $\bx'$ are said to be {\em adjacent}, denoted $\bx \sim \bx'$, if, loosely speaking, they ``differ in one record''; in particular, we use the ``zeroing-out'' adjacency as defined shortly. We consider the following notion of \emph{$(\eps, \delta)$-differential privacy (DP)}.
\begin{definition}[DP]\label{def:dp}
For $\eps, \delta \ge 0$, a mechanism $\cM$ satisfies $(\eps, \delta)$-$\DP$ if for all adjacent datasets $\bx \sim \bx'$, and for any (measurable) event $E$ it holds that
$
\Pr[\cM(\bx) \in E] ~\le~ e^{\eps} \Pr[\cM(\bx') \in E] + \delta.
$
\end{definition}
Following \cite{kairouz21practical,chua24private}, we use the ``zeroing-out'' adjacency. Consider the augmented input space $\cX_{\bot} := \cX \cup \{\bot\}$ and extend any adaptive query method $\cA$ as $\cA(g_1, \ldots, g_t; \bot) = \mathbf{0}$ for all $g_1, \ldots, g_t \in \R^d$. Datasets $\bx, \bx' \in \cX_{\bot}^n$ are said to be {\em zero-out} adjacent if there exists $i$ such that $\bx_{-i} = \bx'_{-i}$, and exactly one of $\{x_i, x_i'\}$ is in $\cX$ and the other is $\bot$. We use $\bx \to_z \bx'$ to specifically emphasize that $x_i \in \cX$ and $x'_i = \bot$. Thus, $\bx \sim \bx'$ if $\bx \to_z \bx'$ or $\bx' \to_z \bx$.

\paragraph{Hockey Stick Divergence \& Dominating Pairs.}%
For probability densities $P$ and $Q$, we use $\alpha P + \beta Q$ to denote the weighted sum of the corresponding densities.
$P \otimes Q$ denotes the product distribution sampled as $(u, v)$ for $u \sim P$, $v \sim Q$, and, $P^{\otimes T}$ denotes the $T$-fold product distribution $P \otimes \cdots \otimes P$.

For all $\eps \in \R$, the {\em $e^\eps$-hockey stick divergence} between $P$ and $Q$ is
$\Deps(P ~\|~ Q) := \sup_{\event} P(\event) - e^{\eps} Q(\event)$.

It is immediate to see that a mechanism $\cM$ satisfies $(\eps, \delta)$-$\DP$ iff for all adjacent $\bx \sim \bx'$, it holds that $\Deps(\cM(\bx) ~\|~ \cM(\bx')) \le \delta$.

\begin{definition}[Dominating Pair~\citep{zhu22optimal}]\label{def:dominating-pair}
The pair $(P, Q)$ {\em dominates} the pair $(A, B)$ (denoted $(P, Q) \dominates (A, B)$) if $\Deps(P ~\|~ Q) ~\ge~ \Deps(A ~\|~ B)$ holds for all $\eps \in \R$.\footnote{Note: this includes $\eps < 0$.} For any mechanism $\cM$,
\begin{itemize}
\item $(P, Q)$ {\em dominates} a mechanism $\cM$ (denoted $(P, Q) \dominates \cM$) if $(P, Q) \dominates (\cM(\bx), \cM(\bx'))$ for all adjacent $\bx \to_z \bx'$.
\item Conversely, $(P, Q)$ is {\em dominated by} $\cM$ (denoted $\cM \dominates (P, Q)$) if there exists $\bx \to_z \bx'$ such that $(\cM(\bx), \cM(\bx')) \dominates (P, Q)$.
\item $(P, Q)$ {\em tightly dominates} $\cM$ (denoted $(P, Q) \equiv \cM$) if $(P, Q) \dominates \cM$ and $\cM \dominates (P, Q)$.
\end{itemize}
\end{definition}

If $(P, Q) \dominates \cM$, then for all $\eps \ge 0$, it holds that $\delta_{\cM}(\eps) \le \max \{ D_{e^\eps}(P ~\|~ Q), \Deps(Q ~\|~ P) \}$,
and conversely, if $\cM \dominates (P, Q)$, then for all $\eps \ge 0$, it holds that $\delta_{\cM}(\eps) \ge \max \{ D_{e^\eps}(P ~\|~ Q), \Deps(Q ~\|~ P) \}$.\footnote{This uses that if $(P, Q) \dominates (A, B)$ then $(Q, P) \dominates (B, A)$, which follows e.g. from \citet[Lemma 46]{zhu22optimal}.}
Consequently, if $(P, Q) \equiv \cM$, then $\delta_{\cM}(\eps) = \max\{ D_{e^\eps}(P ~\|~ Q), \Deps(Q ~\|~ P) \}$.
Thus, tightly dominating pairs completely characterize the privacy loss of a mechanism (although they are not guaranteed to exist for all mechanisms).\footnote{Our terminology is slightly different from \cite{zhu22optimal} in that they refer to $(P, Q)$ as a {\em tightly dominating pair} if $\delta_{\cM}(\eps) = D_{e^{\eps}}(P \| Q)$ for all $\eps \in \R$. Such a notion of a tightly dominating pair always exists, but need not correspond to a worst case pair of adjacent datasets. This is true, e.g. for even one step of $\ABLQP$. Our notation is asymmetric in defining a dominating pair, allowing a tight characterization of $\ABLQP$.}

For a distribution $P$ over $\Omega$, and a randomized function $f : \Omega \to \Gamma$, let $f(P)$ denote the distribution of $f(x)$ for $x \sim P$. The post-processing property of DP implies the following.

\begin{lemma}\label{lem:post-process-and-domination}
For distributions $P, Q$ over $\Omega$, and distributions $A, B$ over $\Gamma$, if there exists $f : \Omega \to \Gamma$ such that simultaneously $f(P) = A$ and $f(Q) = B$ then $(P, Q) \dominates (A, B)$.\footnote{More strongly, the converse is also true (\Cref{lem:converse-post-process-and-domination}).}
\end{lemma}

{\bf \boldmath Dominating Pairs for $\ABLQG$.}
Recall again, that we use $\cD$, $\cP$, $\cS$ to refer to the deterministic, Poisson, and shuffle batch generators respectively.
Tightly dominating pairs are known for $\ABLQD$ and $\ABLQP$:

\begin{proposition}[{\citet[Theorem 8]{balle18improving}}]\label{prop:D-hockey}
For all $\sigma > 0$ and $T \ge 1$, it holds that $(\PD, \QD) \equiv \ABLQD$ for
$\PD \textstyle~:=~ \cN(1, \sigma^2)$ and $\QD \textstyle~:=~ \cN(0, \sigma^2)$.
\end{proposition}

\begin{proposition}[\cite{koskela2020computing,zhu22optimal}]\label{prop:P-dominating}
For all $\sigma > 0$ and $T \ge 1$, it holds that $(\PP, \QP) \equiv \ABLQP$ for
\begin{align*}
\PP &\textstyle~:=~ \prn{\prn{1 - \frac1T} \cN(0, \sigma^2) + \frac1T \cN(1, \sigma^2)}^{\otimes T}\,,\\
\QP &\textstyle~:=~ \cN(0, \sigma^2)^{\otimes T}\,.
\end{align*}
\end{proposition}
These tightly dominating pairs enable numerical computation of the privacy parameters. Namely, $\deltaD(\eps)$ can be computed easily using Gaussian CDFs, whereas $\deltaP(\eps)$ can be computed using numerical methods with the Fast-Fourier transform, as provided in multiple open source libraries~\citep{DPBayes,MicrosoftDP,GoogleDP}.

On the other hand, while it is unclear whether a tightly dominating pair for $\ABLQS$ even exists, \cite{chua24private} studied a pair that is dominated by $\ABLQS$, while conjecturing it to be tightly dominating.
\begin{proposition}[\cite{chua24private}]\label{prop:S-dominated}
For all $\sigma > 0$ and $T \ge 1$, it holds that $\ABLQS \dominates (\PS, \QS)$ for
\begin{align*}
\PS &\textstyle~:=~ \sum_{i=1}^T \frac1T \cdot \cN(2e_i, \sigma^2 I_T)\,,\\
\QS &\textstyle~:=~ \sum_{i=1}^T \frac1T \cdot \cN(e_i, \sigma^2 I_T)\,.
\end{align*}
where $e_i \in \R^T$ is the $i$th standard basis vector, and $I_T$ is the $T \times T$ identity matrix.
\end{proposition}

\section{BALLS-AND-BINS SAMPLER}\label{sec:bnb}
\vspace{-2mm}

We identify a tightly dominating pair for $\ABLQB$.

\begin{theorem}\label{thm:bnb-dominating-pair}
For all $\sigma > 0$ and $T \ge 1$, it holds that $(\PB, \QB) \equiv \ABLQB$ for\footnote{Incidentally, the same pair $(\PB, \QB)$ was studied in the context of shuffling by \cite{koskela23numerical}, who also discuss the difficulty of numerically approximating its hockey stick divergence.}
\[\textstyle
\PB ~:=~ \sum_{t=1}^T \frac1T \cdot \cN(e_t, \sigma^2 I_T)\,,\quad 
\QB ~:=~ \cN(0, \sigma^2 I_T)\,.
\]
\end{theorem}

The proof relies on the well-known ``joint convexity'' property of the hockey stick divergence.

\begin{proposition}[Joint Convexity of Hockey Stick Divergence; see, e.g., Lemma B.1 in \cite{chua24private}]\label{prop:joint-convexity}
Given two mixture distributions $P = \sum_{i=1}^m \alpha_i P_i$ and $Q = \sum_{i=1}^m \alpha_i Q_i$, it holds for all $\eps \in \R$ that
$D_{e^\eps}(P ~\|~ Q) ~\le~ \sum_i \alpha_i D_{e^\eps}(P_i ~\|~ Q_i)\,.$
\end{proposition}

\begin{proof}[Proof of \Cref{thm:bnb-dominating-pair}]
To show $\ABLQB \dominates (\PB, \QB)$, consider $\cX = [-1, 1]$ and let $\bx = (0, \ldots, 0, 1)$ and $\bx' = (0, \ldots, 0, \bot)$. Consider
$\cA$ that always generates the query $\psi_t(x) = x$ (and by definition $\psi_t(\bot) = 0$). In this case, it is immediate that $\PB = \ABLQB(\bx)$ and $\QB = \ABLQB(\bx')$.

To show that $(\PB, \QB) \dominates \ABLQB$, consider any adjacent datasets $\bx \to_z \bx'$ that differ on say example $x_n$ and $x_n' = \bot$.
For $\Gamma = (S_1', \ldots, S_T')$ for $S_i' \subseteq [n-1]$ being an assignment of batches for all other examples, let $\cB_{\Gamma}$ refer to the batch generator that samples $t$ uniformly in $[T]$ and returns $S_t = S_t' \cup \{n\}$ and $S_{r} = S_r'$ for all $r \ne t$.
We provide a post-processing function $f$ such that $f(\PB) = \ABLQ_{\cB_{\Gamma}}(\bx)$ and $f(\QB) = \ABLQ_{\cB_{\Gamma}}(\bx')$ for any $\Gamma$.
Since $\ABLQB(\cdot) = \frac{1}{T^{n-1}}\sum_{\Gamma} \ABLQ_{\cB_\Gamma}(\cdot)$, it follows from \Cref{lem:post-process-and-domination,prop:joint-convexity} that $(\PB, \QB) \dominates \ABLQB$.

Consider the randomized post-processing function $f$ that maps vectors $v \in \R^T$ to $(\R^d)^T$ (output space of $\ABLQB$) as follows: Given vector $(v_1, \ldots, v_T) \in \R^T$, let $g_t ~=~ \sum_{i \in S_t'} \psi_t(x_i) + \psi_t(x_n) \cdot v_t + e_t$
for $e_t \sim \cN(0, \sigma^2 (I - \psi_t(x_n) \psi_t(x_n)^\top))$; this is inductively defined since $\psi_t$ potentially depends on $g_1, \ldots, g_{t-1}$.

First let us consider the case when $v \sim \QB = \cN(0, \sigma^2 I)$. In this case, $\psi_t(x_n) v_t + e_t$ is distributed precisely as $\cN(0, \sigma^2 I)$, and thus, $g_t$ is distributed as $\sum_{i \in S_t'} \psi_t(x_i) + e'_t$ for $e'_t \sim \cN(0, \sigma^2 I)$ precisely as in $\ABLQB(\bx')$, since $\psi_t(x_n') = 0$ so it does not matter which batch $x_n'$ lands in. Thus, $f(\QB) = \ABLQB(\bx')$.

On the other hand, $v \sim \PB$, is equivalent to sampling $t_*$ uniformly at random in $[T]$, and sampling $v \sim \cN(e_{t_*}, \sigma^2 I)$. For a fixed $t_*$ and sampling $v \sim \cN(e_{t_*}, \sigma^2 I)$, $g_t$ is distributed as $\sum_{i \in S_t'} \psi_t(x_i) + e'_t$ for $e'_t \sim \cN(0, \sigma^2 I)$ for $t \ne t_*$ and distributed as $\sum_{i \in S_t'} \psi_t(x_i) + \psi_t(x_n) + e'_t$ for $e'_t \sim \cN(0, \sigma^2 I)$ for $t = t_*$, which is precisely as in $\ABLQB(\bx)$ when $x_n$ lands in batch $S_{t_*}$. Since $t_*$ is uniformly random in $[T]$,
we get that $f(\PB) = \ABLQB(\bx)$.
\end{proof}
We now compare the privacy of $\ABLQB$ against $\ABLQG$ for $\cG \in \{\cD, \cS, \cP\}$.

{\bf\boldmath $\deltaB$ vs. $\deltaD$ and $\deltaS$.}
A simple consequence of \Cref{prop:joint-convexity} is that $\ABLQB$ and $\ABLQS$ have better privacy guarantees than $\ABLQD$, since any fixed assignment of examples to batches results in the analysis being equivalent to $\ABLQD$.
In fact, more strongly, we observe that in fact $\ABLQB$ always has better privacy guarantees than $\ABLQS$.
\begin{proposition}\label{prop:B-vs-S-and-D}
For all $\eps > 0$, $\sigma > 0$, and $T \ge 1$, it holds that $\deltaB(\eps) \le \deltaS(\eps) \le \deltaD(\eps)$.
\end{proposition}
\begin{proof}
Consider the same case as in the proof of \Cref{thm:bnb-dominating-pair} where $\cX = [-1, 1]$ and let $\bx = (0, \ldots, 0, 1)$ and $\bx' = (0, \ldots, 0, \bot)$. Consider
$\cA$ that always generates the query $\psi_t(x) = x$ (and by definition $\psi_t(\bot) = 0$). In this case, it is immediate to see that $\PB = \ABLQB(\bx) = \ABLQS(\bx)$ and $\QB = \ABLQB(\bx') = \ABLQS(\bx')$. Thus, we get that $\ABLQS \dominates (\PB, \QB)$ and hence $\deltaB(\eps) \le \deltaS(\eps)$.
\end{proof}

{\bf\boldmath $\deltaB$ vs. $\deltaP$.}
Finally, we observe that $\ABLQB$ has better privacy guarantees than $\ABLQP$ at large $\eps$.

\begin{theorem}\label{thm:bnb-vs-poisson}
For all $\sigma > 0$ and $T > 1$, there exist $\eps_0 > 0$ such that, $\forall \eps > \eps_0$, it holds that $\deltaB(\eps) < \deltaP(\eps)$.
\end{theorem}
\begin{proof}%
\citet[Theorem 4.2]{chua24private} showed that there exists an $\eps_0$ such that for all $\eps > \eps_0$, it holds that $\deltaD(\eps) < \deltaP(\eps)$. Combining this with \Cref{prop:B-vs-S-and-D} gives us that $\deltaB(\eps) < \deltaP(\eps)$ for all $\eps \ge \eps_0$.
\end{proof}

\begin{remark}\label{rem:bnb-poisson-incomparable}
In \Cref{app:bnb-vs-poisson-incomparable}, we show that in fact for all $\sigma > 0$ and $T > 1$, $(\PB, \QB) \not\dominates (\PP, \QP)$; the reverse direction $(\PP, \QP) \not\dominates (\PB, \QB)$ is already implied by \Cref{thm:bnb-vs-poisson}. Thus, the privacy guarantees of $\ABLQP$ and $\ABLQB$ are in general incomparable.
\end{remark}

\section{\boldmath ESTIMATING \texorpdfstring{$\deltaB(\eps)$}{deltaB(eps)}}\label{sec:Monte Carlo}
\vspace{-2mm}

Hockey stick divergence $D_{\eps}(P ~\|~ Q)$ can be expressed in terms of the {\em privacy loss function} $L_{P ~\|~ Q}(x)$~\citep{dwork16concentrated}\footnote{The distribution of $L_{P ~\|~ Q}(x)$ for $x \sim P$ is also known as the {\em privacy loss random variable}. However for later convenience we use the {\em loss function} terminology.} as
\begin{align}\textstyle
    \Deps(P ~\|~ Q) = \Ex_{x \sim P}~[1 - e^{\eps - L_{P ~\|~ Q}(x)}]_+,\label{eq:hsd-as-expectation}
\end{align}
where $L_{P ~\|~ Q}(x) := \log \frac{P(x)}{Q(x)}$,\footnote{Technically speaking, $\frac{P(x)}{Q(x)}$ should be replaced by the Radon--Nikodym derivative of $\frac{dP}{dQ}(x)$, but for purposes of this work, it suffices to consider the case of densities.} where $P(x)$ and $Q(x)$ refers to the density of $P$ and $Q$ at $x$ and $[z]_+ := \max\{0, z\}$. The privacy loss function $L_{\PB ~\|~ \QB}(x)$ for the pair $(\PB, \QB)$ and $(\QB, \PB)$ at $x \in \R^T$ are as follows:
\begin{align}
L_{\PB ~\|~ \QB}(x) &\textstyle~:=~ \log \frac{\PB(x)}{\QB(x)}
~=~ \log \frac{\sum_{t=1}^T e^{-\| x - e_t\|^2 / (2\sigma^2)}}{T \cdot e^{-\|x\|^2 / (2\sigma^2)}} \nonumber\\
&\textstyle~=~ \log \prn{\sum_{t=1}^T e^{x_t / \sigma^2}} - \log T - \frac{1}{2\sigma^2},\label{eq:loss_pb_qb}\\
L_{\QB ~\|~ \PB}(x)
&\textstyle~=~ \log T + \frac{1}{2\sigma^2} - \log \prn{\sum_{t=1}^T e^{x_t / \sigma^2}}.\label{eq:loss_qb_pb} 
\end{align}

{\bf Monte Carlo Estimation.} \eqref{eq:hsd-as-expectation} suggests a natural approach for estimating $\Deps(P ~\|~ Q)$, via drawing multiple samples $x \sim P$ and returning the average value of $L_{P ~\|~ Q}(x)$; such an approach was previously studied by \cite{wang23randomized}. We can obtain high probability upper bounds via the Chernoff--Hoeffding bound as described in \Cref{alg:mce}, wherein $\KL(q ~\|~ p) := q \log \frac{q}{p} + (1 - q) \log \frac{1-q}{1-p}$ is the KL divergence for Bernoulli random variables $\Ber(p)$ and $\Ber(q)$.

\begin{theorem}\label{thm:high-prob-delta}
For all distributions $P$ and $Q$, \Cref{alg:mce} returns an upper bound on $\Deps(P ~\|~ Q)$ with probability $1 - \beta$.
\end{theorem}

\begin{fact}[\cite{hoeffding63probability}]\label{fact:hoeffding}
For $Z_1, \ldots, Z_m$ drawn i.i.d. from distribution $P$ over $[0, 1]$ with mean $\mu$, %
\[\textstyle
\Pr\sq{\frac1m \sum_{i=1}^m Z_i < q} \le e^{-\KL(q ~\|~ \mu) m}.
\]
\end{fact}
\begin{proof}[Proof of \Cref{thm:high-prob-delta}]
Suppose $\Deps(P, Q) = \mu$. If the returned value $p$ by \Cref{alg:mce} is $1$, then $\mu \le 1$ holds trivially. Otherwise if $p$ is smaller than $\mu$, then we have that the empirical average $q$ is such that $\KL(q ~\|~ \mu) > \KL(q ~\|~ p) \ge \log(1/\beta) / m$. From \Cref{fact:hoeffding}, this can happen with probability at most $\beta$.
\end{proof}
\Cref{alg:mce} in principle allows us to obtain an upper bound on the hockey stick divergence to arbitrary accuracy and with arbitrarily high probability through the guarantees in \Cref{thm:high-prob-delta} as the number of samples $m\to\infty$.
However, there are two challenges when implementing it in practice. First, the sample size $m$ needed can be quite large if we want a small multiplicative approximation in a regime where $\Deps(P ~\|~ Q)$ is very small. Furthermore in the case of $(\PB, \QB)$, sampling each $x^{(i)}$, which is $T$ dimensional, is computationally intensive for large $T$. We tackle each of these challenges using importance sampling and a new method of ``order statistics sampling'' respectively; we believe the latter could be of independent interest.

Before describing these methods, we first note a simplification when estimating $\Deps(\PB ~\|~ \QB)$ using \Cref{alg:mce}. For $P_t := \cN(e_t, \sigma^2 I)$, we have that $\PB = \sum_{t=1}^T \frac1T P_t$. By symmetry of $L_{\PB ~\|~ \QB}(x)$, it follows that
\begin{align*}
\Deps(\PB ~\|~ \QB)
&\textstyle~=~ \frac1T \sum_{t=1}^T \E_{x \sim P_t} [1 - e^{\eps - L_{\PB ~\|~ \QB}(x)}]_+\\
&\textstyle~=~ \E_{x \sim P_1} [1 - e^{\eps - L_{\PB ~\|~ \QB}(x)}]_+
\end{align*}
Thus, it suffices to sample $x^{(i)} \sim P_1$ instead of $\PB$ in \Cref{alg:mce}.

\begin{algorithm}[t]
\caption{Monte Carlo Estimation of $\Deps(P ~\|~ Q)$.}
\label{alg:mce}
\begin{algorithmic}
\REQUIRE Distributions $P$ and $Q$; sample access to $P$
\STATE\phantom{{\bf Inpu}} Sample size $m$, Error probability $\beta$.
\ENSURE An upper confidence bound on $\Deps(P ~\|~ Q)$.
\STATE Sample $x^{(1)}, \ldots, x^{(m)} \sim P$
\STATE $q \gets \frac1m \sum_{i=1}^m \max\{ 0, 1 - e^{\eps - L_{P ~\|~ Q}(x^{(i)})}\}$
\STATE $p \gets$ smallest value in $[q, 1]$ such that $\KL(q ~\|~ p) \ge \log(1/\beta) / m$, or $1$ if no such value exists
\RETURN $p$
\end{algorithmic}
\end{algorithm}

{\bf Importance Sampling.} We provide a generic approach for improving the sample complexity of \Cref{alg:mce} via importance sampling. For any $P$ and $Q$, let $E_{\eps}$ be any event such that $L_{P ~\|~ Q}(x) < \eps$ for all $x \notin E_{\eps}$. Then, it suffices to ``zoom in'' on $E_{\eps}$ by sampling from the conditional distribution $P|_{x \in E_{\eps}}$, when estimating the expectation in \eqref{eq:hsd-as-expectation}, as explained in \Cref{alg:importance-mce}. Since this requires sample access to $P|_{x \in E_{\eps}}$, the set $E_{\eps}$ has to be chosen carefully so that it will be efficient to sample from $P|_{x \in E_{\eps}}$.
We now define the specific sets $E_{\eps}$ we use for estimating $\Deps(\PB ~\|~ \QB)$ and $\Deps(\QB ~\|~ \PB)$,  starting with the latter.

For estimating $\Deps(\QB ~\|~ \PB)$, let $E_\eps := \{ x \in \R^T : \max_{t \in [T]} x_t \le C_\eps \}$ where $C_{\eps} = \frac12 - \eps \sigma^2$.
The choice of $C_\eps$ is such that for all $x \notin E_{\eps}$,
\begin{align*}
\textstyle L_{\QB ~\|~ \PB}(x) &\textstyle~=~ \log T + \frac{1}{2\sigma^2} - \log\prn{\sum_{t=1}^T e^{x_t / \sigma^2}}\\
&\textstyle~\le~ \log T + \frac{1}{2\sigma^2} - \log\prn{T e^{C_\eps / \sigma^2}}\\
&\textstyle~=~ \frac1{2\sigma^2} - \frac{C_\eps}{\sigma^2} ~=~ \eps.
\end{align*}
We show how to efficiently sample from $\QB|_{x \in E_{\eps}}$ in \Cref{subapp:importance}.

For estimating $\Deps(\PB ~\|~ \QB)$, we consider $E_\eps := \{ x : \max\{ x_1 - 1, \max_{t > 1} x_t\} \ge C_\eps \}$ such that
\[\textstyle
C_{\eps} ~=~ \frac{1}{2} + \sigma^2 \cdot \prn{\eps - \log\prn{1 + \frac{e^{1/\sigma^2} - 1}{T}}}.
\]
The choice of $C_{\eps}$ implies that for all $x \notin E_{\eps}$, it holds that $x_1 \le C_{\eps} + 1$ and $x_t \le C_{\eps}$ for all $t > 1$, and hence
\begin{align*}
&\textstyle L_{\PB ~\|~ \QB}(x) ~=~ \log\prn{\sum_{t=1}^T e^{x_t / \sigma^2}} - \log T - \frac{1}{2\sigma^2}\\
&\textstyle~\le~ \log(e^{(C_{\eps}+1) / \sigma^2} + (T-1) e^{C_{\eps} / \sigma^2}) - \log T - \frac{1}{2\sigma^2}\\
&\textstyle~=~ \frac{C_\eps}{\sigma^2} + \log(e^{1 / \sigma^2} + T-1) - \log T - \frac{1}{2\sigma^2} ~\le~ \eps.
\end{align*}
Note that as before we sample from $P_1|_{x \in E_{\eps}}$ instead of $\PB|_{x \in E_{\eps}}$. Again, it is efficient to sample from $P_1|_{x \in E_{\eps}}$, and we defer the details to \Cref{subapp:importance}.

In general, this approach for importance sampling reduces the sample complexity by a factor of $1 / P(E_\eps)$, and we numerically demonstrate the improved sample complexity for specific examples in \Cref{subapp:importance}.

\begin{algorithm}[t]
\caption{Monte Carlo Estimation of $\Deps(P ~\|~ Q)$ with Importance Sampling.}
\label{alg:importance-mce}
\begin{algorithmic}
\REQUIRE Distributions $P$ and $Q$,
\STATE\phantom{{\bf Inpu}} Event $E$ such that $L_{P ~\|~ Q}(x) < \eps$ for all $x \notin E$,
\phantom{{\bf Inpu}} Sample access to $P|_{x \in E}$,
\STATE\phantom{{\bf Inpu}} Sample size $m$, Error probability $\beta$.
\ENSURE An upper confidence bound on $\Deps(P ~\|~ Q)$.
\STATE Sample $x^{(1)}, \ldots, x^{(m)} \sim P|_{x \in E}$
\STATE $q \gets \frac1m \sum_{i=1}^m \max\{ 0, 1 - e^{\eps - L_{P ~\|~ Q}(x^{(i)})}\}$
\STATE $p \gets$ smallest value in $[q, 1]$ such that $\KL(q ~\|~ p) \ge \log(1/\beta) / m$, or $1$ if no such value exists
\RETURN $p \cdot P(E)$
\end{algorithmic}
\end{algorithm}

{\bf Order Statistics Sampling.}
As mentioned before, sampling $x \sim P_1$ or $x \sim \QB$ can be slow when the number of steps $T$ is large, especially since we also need to draw a large sample of size $m$. To make this efficient, at a slight cost of obtaining pessimistic bounds on $\Deps(\PB ~\|~ \QB)$ and $\Deps(\QB ~\|~ \PB)$, we use the following approximation: For any list $k_1, \ldots, k_r \in \{1, \ldots, R\}$  of increasing indices, the following holds, where we use $k_{r+1} := R$ and $k_0 := 0$ for convenience. Given $x_1, \ldots, x_R \in \R$, let $y^{(1)}, \ldots, y^{(R)}$ denote the same values as $x_i$'s but in sorted order $y^{(1)} \ge \dots \ge y^{(R)}$. Then,
\begin{align}
\textstyle\sum_{t=1}^R e^{x_t / \sigma^2}
&\textstyle~\le~ \sum_{i=1}^r (k_{i+1} - k_{i}) \cdot e^{y^{(k_i)} / \sigma^2}\label{eq:slimane-ub}\\
\textstyle\sum_{t=1}^R e^{x_t / \sigma^2}
&\textstyle~\ge~ \sum_{i=1}^r (k_{i} - k_{i-1}) \cdot e^{y^{(k_i)} / \sigma^2},\label{eq:slimane-lb}
\end{align}
where \eqref{eq:slimane-ub} additionally requires that $k_1 = 1$. These approximations are inspired by and are a generalization of the bounds introduced by \cite{slimane01bounds}, which correspond to $k_1 = 1$ and $k_2 = 2$ for \eqref{eq:slimane-ub} and $k_1 = 1$ and $k_2 = R$ for \eqref{eq:slimane-lb}.

Our key idea for efficient sampling is to directly sample from the joint distribution of order statistics $y^{(k_1)}, \ldots, y^{(k_r)}$. In particular, for any distribution $P$ with efficiently computable inverse of the cumulative density function $\CDF_P^{-1}(\cdot)$, \Cref{alg:order-sampling} efficiently samples order statistics of $P$, wherein $\Beta(a, b)$ is the Beta distribution over $[0, 1]$. \Cref{thm:order-sampling} (\Cref{subapp:order-stats}) establishes the correctness of this algorithm.

When estimating $\Deps(\QB ~\|~ \PB)$ we sample an upper bound on $L_{\QB ~\|~ \PB}(x)$ as
\[\textstyle
\log T + \frac1{2\sigma^2} - \log\prn{\sum_{i=1}^r (k_{i} - k_{i-1}) \cdot e^{y^{(k_i)} / \sigma^2}},
\]
for $R = T$, and when estimating $\Deps(\PB ~\|~ \QB)$ we sample an upper bound on $L_{\PB ~\|~ \QB}(x)$ as
\[\textstyle
\log\prn{e^{x_1 / \sigma^2} + \sum_{i=1}^r (k_{i+1} - k_{i}) \cdot e^{y^{(k_i)} / \sigma^2}} - \log T - \frac1{2\sigma^2},
\]
where we sample $x_1 \sim \cN(1, \sigma^2)$ and $y^{(k_1)}, \ldots, y^{(k_r)}$ using \Cref{alg:order-sampling} for $R = T-1$, and plug them into \Cref{alg:mce}. While not immediate, this can also be used in conjunction with importance sampling as in \Cref{alg:importance-mce}. 
We defer the details to \Cref{subapp:importance-order-combined}.

\begin{algorithm}[t]
\caption{Sampling Order Statistics of $P$}
\label{alg:order-sampling}
\begin{algorithmic}
\REQUIRE Number $R$ of random variables $\sim P$
\STATE\phantom{{\bf Inpu}} Orders $k_1, \ldots, k_r \in \{1, \ldots, R\}$
\ENSURE $(y^{(k_1)}, \ldots, y^{(k_r)})$ jointly distributed as $(k_1, \ldots, k_r)$ order statistics of $R$ draws from $P$ 
\STATE Let $k_0 \gets 0$
\FOR{$i = 1, \ldots, r$}
    \STATE $z_i \sim \Beta(R-k_i+1, k_i - k_{i-1})$
    \STATE $y^{(k_i)} \gets \CDF_P^{-1}(\prod_{j=1}^i z_j)$
\ENDFOR
\RETURN $(y^{(k_1)}, \ldots, y^{(k_r)})$
\end{algorithmic}
\end{algorithm}

{\bf \boldmath Lower Bounds on $\deltaB(\eps)$.}
Finally, in addition to Monte Carlo estimation, we can obtain a lower bound on $\deltaB(\eps)$ that is efficient to compute, inspired by \cite{chua24private}. The idea is to use the following lower bound\vspace{-2mm}
\begin{align}
\deltaB(\eps) \ge \Deps(\PB ~\|~ \QB) \ge \sup_{C} \PB(S_C) - e^{\eps} \QB(S_C)\label{eq:bnb-lb},%
\end{align}
where $S_C := \{x : \max_t x_t \ge C\}$; the main reason being that $\PB(S_C)$ and $\QB(S_C)$ are efficiently computable. We numerically observe that the lower bounds computed by this method are in fact quite close to the Monte Carlo estimates of $\deltaB(\eps)$ found via \Cref{alg:mce}; see, e.g., \Cref{fig:privacy}.

Finally, we note that Monte Carlo estimation can be easily parallelized by having different machines generate samples and then combining the estimates, which can reduce the wall clock time. This allows scaling the Monte Carlo estimation to a large number of samples. Our implementation of privacy accounting is available at \href{https://github.com/google-research/google-research/tree/master/dpsgd_batch_sampler_accounting}{\small github.com/google-research/google-research/tree/master/dpsgd\_batch\_sampler\_accounting}.

\section{EXPERIMENTS}\label{sec:experiments}
\vspace{-2mm}
We compare the utility of $\DPSGD$ using $\cS$, $\cP$, and $\cB$ batch generators.
We compare all algorithms with the same noise scale $\sigma$ to isolate the impact of using different batch samplers from the privacy accounting.

{\bf Implementation Details.}
We use the scalable batch sampling approach proposed by \cite{chua24scalable} using massively parallel computation~\citep{dean04mapreduce} for sampling batches using each of the batch generators.
We use JAX~\citep{jax2018github} for training neural networks.
Since it is more efficient to have fixed batch sizes, we follow \cite{chua24scalable} and fix a certain maximum batch size $B$ when using \BnB and Poisson subsampling, and for batches that exceed size $B$, we randomly subsample without replacement to get a batch of size $B$. This can be done by incurring a small penalty in the privacy parameters, as described in \Cref{app:training}.
For batches that are smaller than size $B$, we pad with  examples with a weight of $0$ so that the batch size is exactly $B$. We use a weighted loss function so that the mini-batch loss is unaffected by the padding.
We note that other optimizations could also be possible, such as using accumulation of gradients computed in small physical batches~\citep{abadi16deep,beltran24towards}.

Finally, we note \BnB can be implemented trivially given any implementation of Shuffling. Namely, given examples $x_1, \ldots, x_n$ in a randomly shuffled order, we construct batches of sizes $b_1, \ldots, b_{T-1}, b_T$ in sequential order where each $b_t$ is inductively sampled from the binomial distribution $\mathrm{Bin}(n - \sum_{i=1}^{t-1} b_i, 1/(T-t+1))$. An alternative approach could be to combinatorially simulate throwing $n$ balls into $T$ bins, to generate the sequence of batch sizes, which while less efficient, potentially avoids floating point errors in the binomial probabilities. Thus, \BnB is similar to implement as Shuffling.

{\bf Datasets.}
The first dataset we use is the Criteo Display Ads pCTR Dataset~\citep{tien14criteokaggle}, which contains around 46M examples. We split the labeled training set from the dataset chronologically into a 80\%/10\%/10\% partition of train/validation/test sets. We consider the task of predicting the probability of a click on an ad from the remaining features.

The second dataset we consider is the Criteo Sponsored Search Conversion Log Dataset~\citep{tallis2018reacting}, which contains 16M examples. We randomly split the dataset into a 80\%/20\% partition of train/test sets. We consider a conversion prediction task, where we predict the binary feature \texttt{Sale}. We omit the features denoted \textit{Outcome/Labels} in~\cite{tallis2018reacting}, and \texttt{product\_price}, which is highly correlated with the label.

For both datasets, we use the binary cross entropy loss for training and report the AUC on the labeled test split, averaged over three runs with each run using independently generated batches. We plot the results with error bars indicating a single standard deviation. For more details about the model architectures and training, see \Cref{app:training}.

{\bf Results.}
We train with $\DPSGD$ with various values of $\sigma$, and for reference, we also train with regular $\SGD$ without any clipping or noise for different batch sizes. The model utilities in terms of AUC are in \Cref{fig:utility}.  The $\sigma$ values we chose were motivated as follows:

(i) First, we consider ``large'' values of $\sigma$, namely in $\{0.1, 0.2, 0.3, 0.4\}$, such that the privacy parameters when using Poisson subsampling / Ball-and-Bins sampling are in a regime that is common in practice (as seen from examples in \cite{desfontainesblog2021list}),%

(ii) We also consider tiny values of $\sigma$ in $\{10^{-2}, 10^{-3}, 10^{-4}, 10^{-5}\}$ to understand how the different samplers behave in the regime interpolating between ``no privacy'' and ``commonly used regimes of privacy''.

We observe that for non-private $\SGD$, \BnB and Shuffling have similar utility and improve significantly over Poisson subsampling. For $\DPSGD$, we observe similar trends for noise multipliers at most $0.001$, but for higher noise multipliers that could be deemed more relevant in practice, the different batch generators all have similar utility.

Next, we plot bounds on $\deltaG(\eps)$ for $\cG \in \{\cS, \cP, \cB\}$ for different combinations of $\sigma$ and (expected) batch size in \Cref{fig:privacy}.
For $\deltaP$, we plot both upper and lower bounds as computed using the \texttt{dp\_accounting} library~\citep{GoogleDP}. For $\deltaS$ we plot a lower bound as shown by \cite{chua24private}. For $\deltaB$, we plot a lower bound from \eqref{eq:bnb-lb}, the mean of the Monte Carlo estimate (value $q$ in \Cref{alg:mce}) and the upper confidence bound (value $p$ in \Cref{alg:mce}) for error probability $\beta = 10^{-3}$. We find even the upper confidence bounds on $\deltaB$ to be lower than $\deltaP$ for the most part, with the exceptional cases when $\deltaP$ is smaller than $10^{-7}$, as this is the region where the concentration bounds are not strong enough. We believe that $\deltaB(\eps) < \deltaP(\eps)$ even in this regime.

\def\figheight{0.25}
\begin{figure*}
  \centering
  \begin{tabular}{cc}
  \includegraphics[height=\figheight\linewidth]{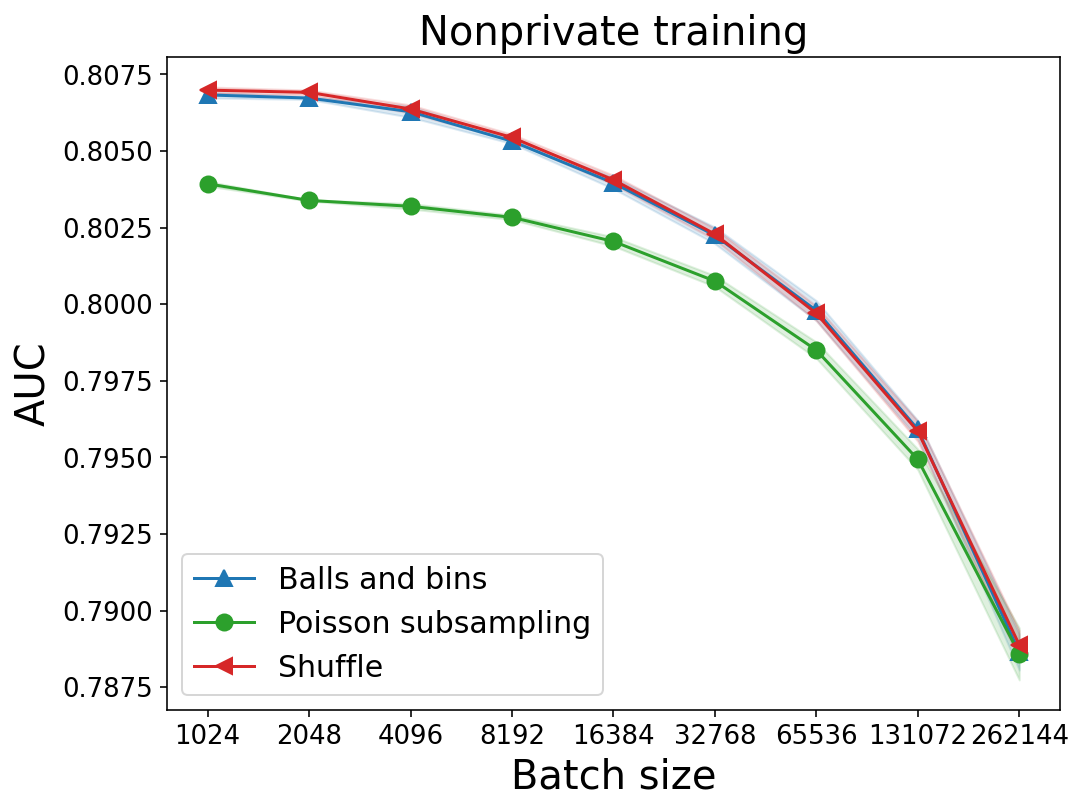} & 
  \includegraphics[height=\figheight\linewidth]{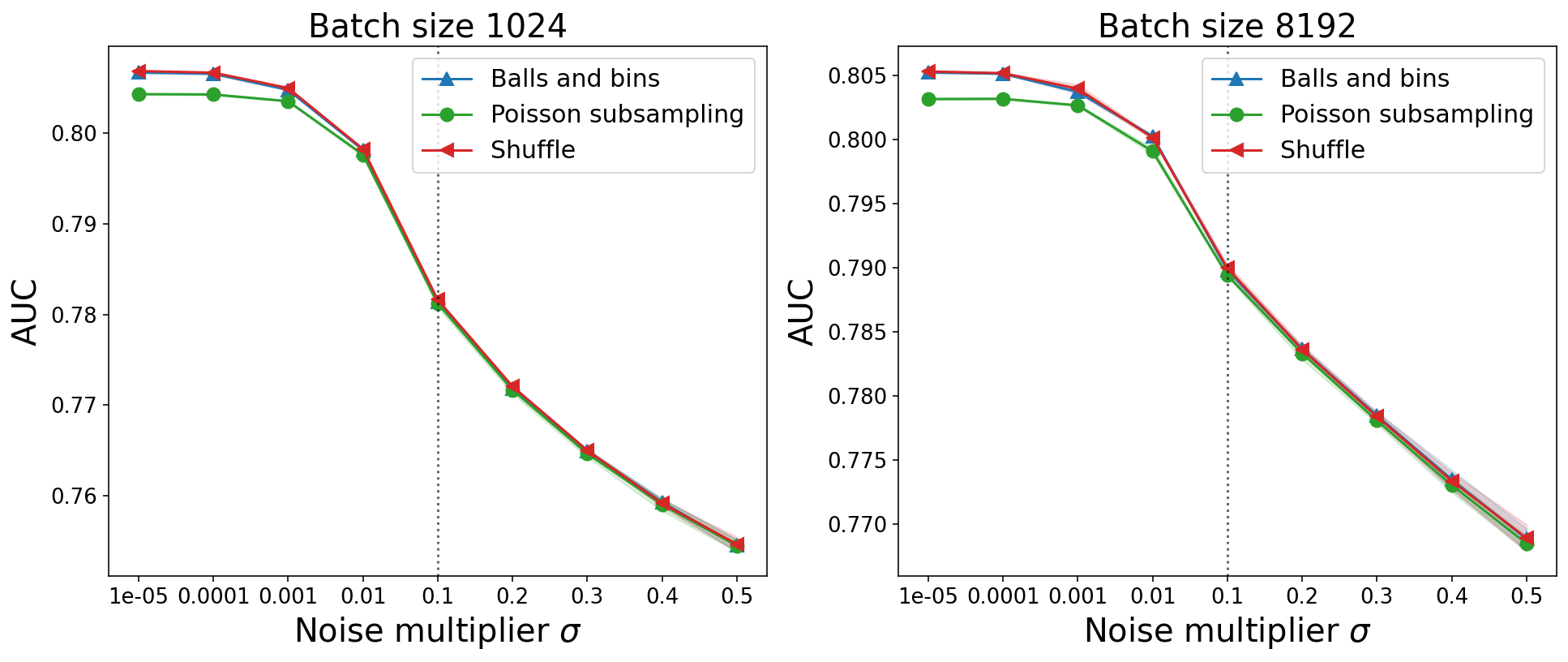} \\
  \includegraphics[height=\figheight\linewidth]{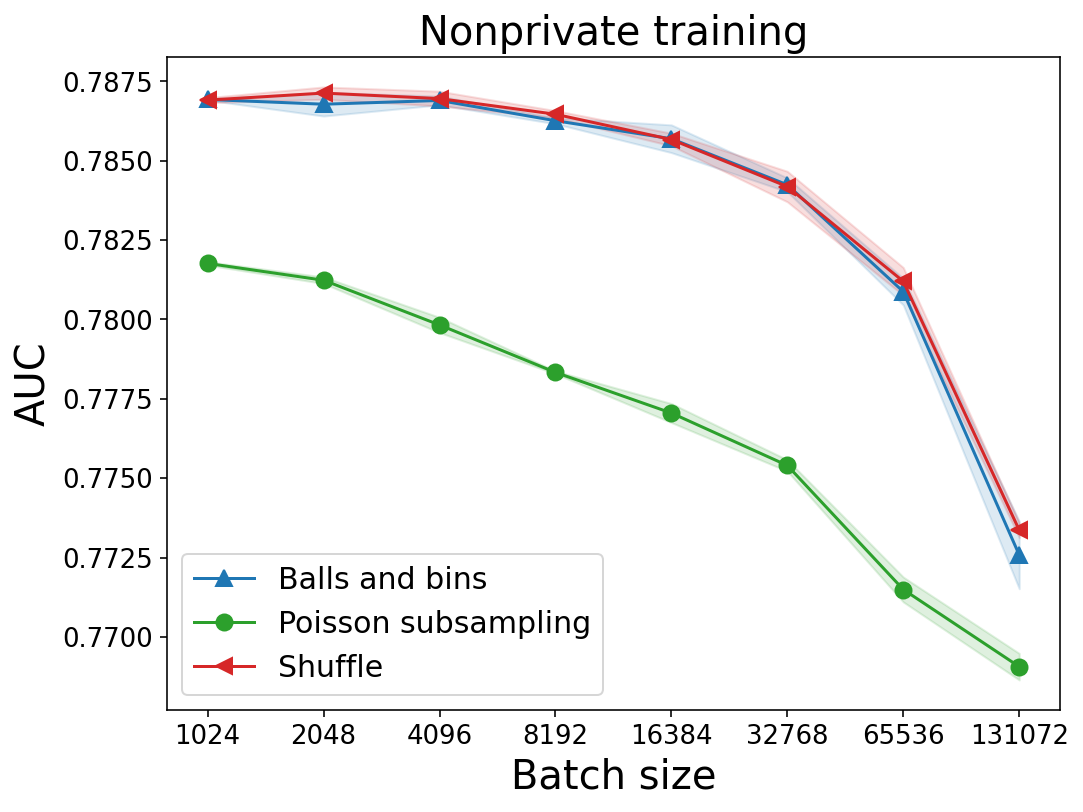} & 
  \includegraphics[height=\figheight\linewidth]{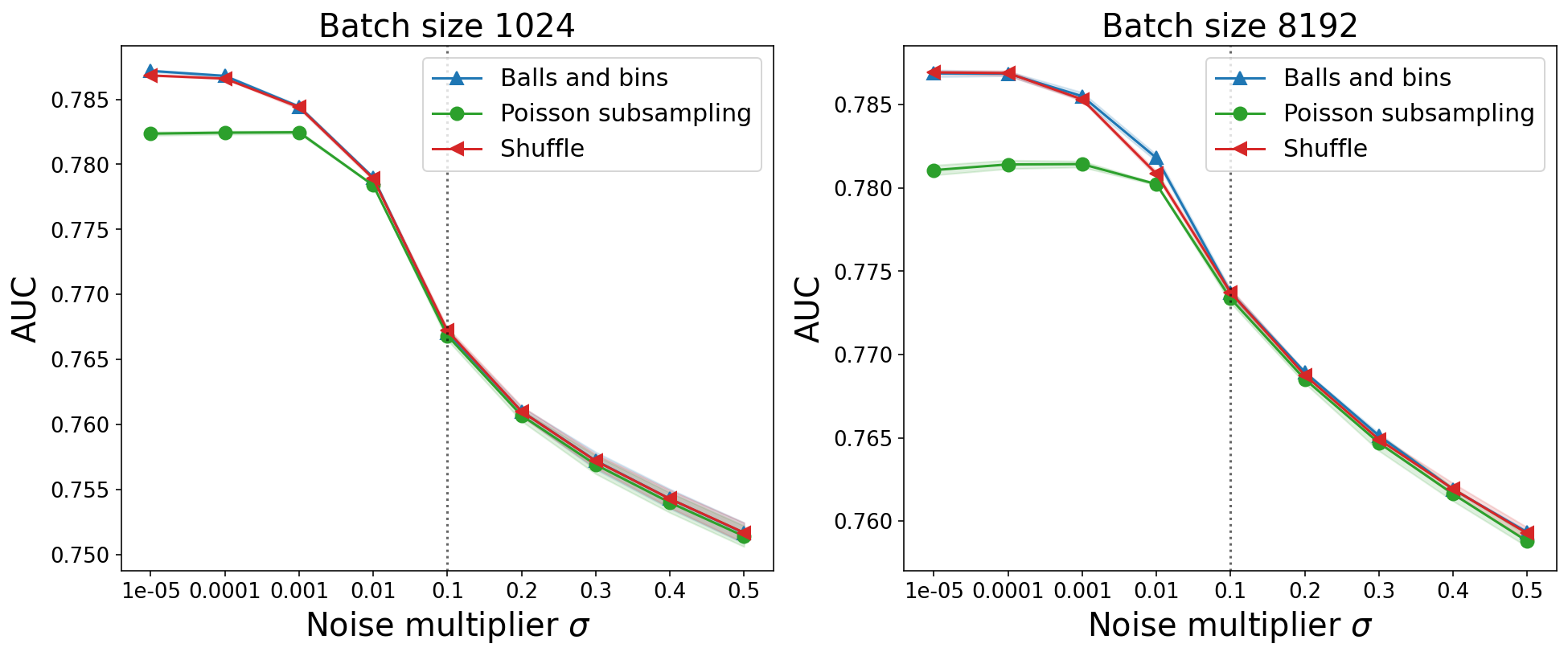} \\
  \end{tabular}
  \caption{AUC values for 1 epoch of training with the Criteo Display Ads pCTR dataset (top) and the Criteo Sponsored Search Conversion Log dataset (bottom). On the left, we train without privacy and vary the batch size. In the middle and right, we train privately with varying $\sigma$, using (expected) batch sizes 1024 (middle) and 8192 (right). We use a log scale to the left of the vertical dotted line at $\sigma = 0.1$, and a linear scale to the right.}
  \label{fig:utility}
\end{figure*}

\def\figheight{0.27}
\begin{figure*}
\centering
\begin{tabular}{cc}
\includegraphics[height=\figheight\linewidth]{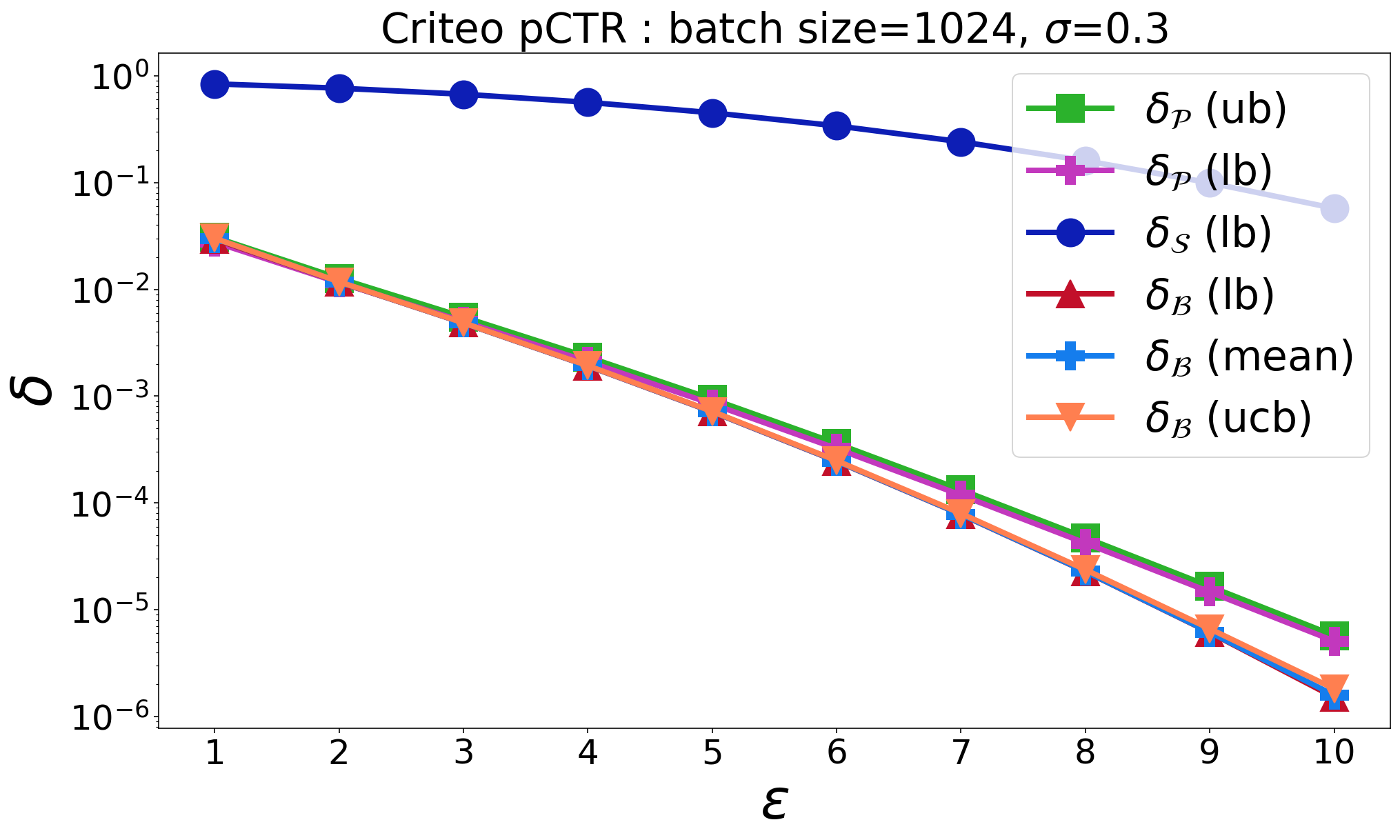} & 
\includegraphics[height=\figheight\linewidth]{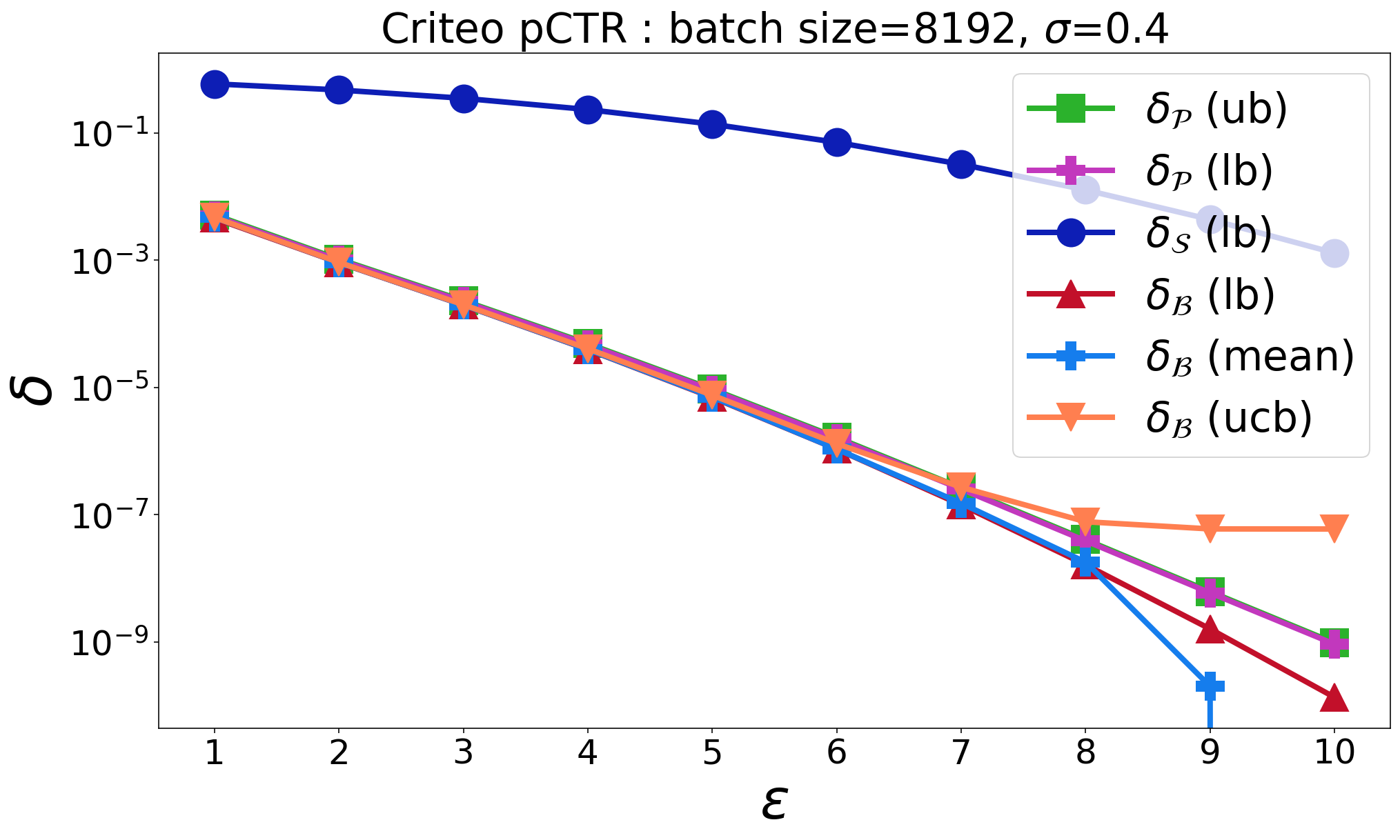} \\
\includegraphics[height=\figheight\linewidth]{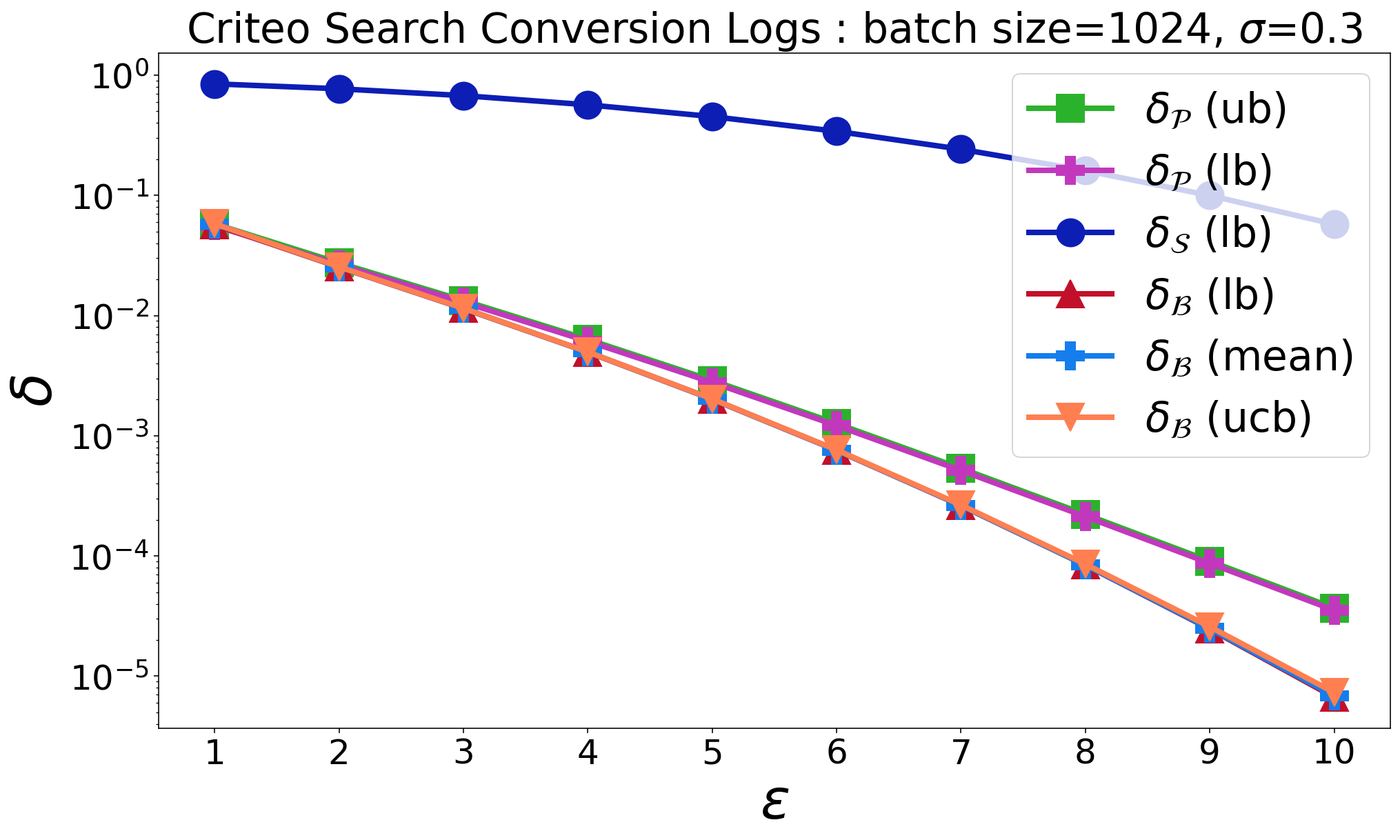} & 
\includegraphics[height=\figheight\linewidth]{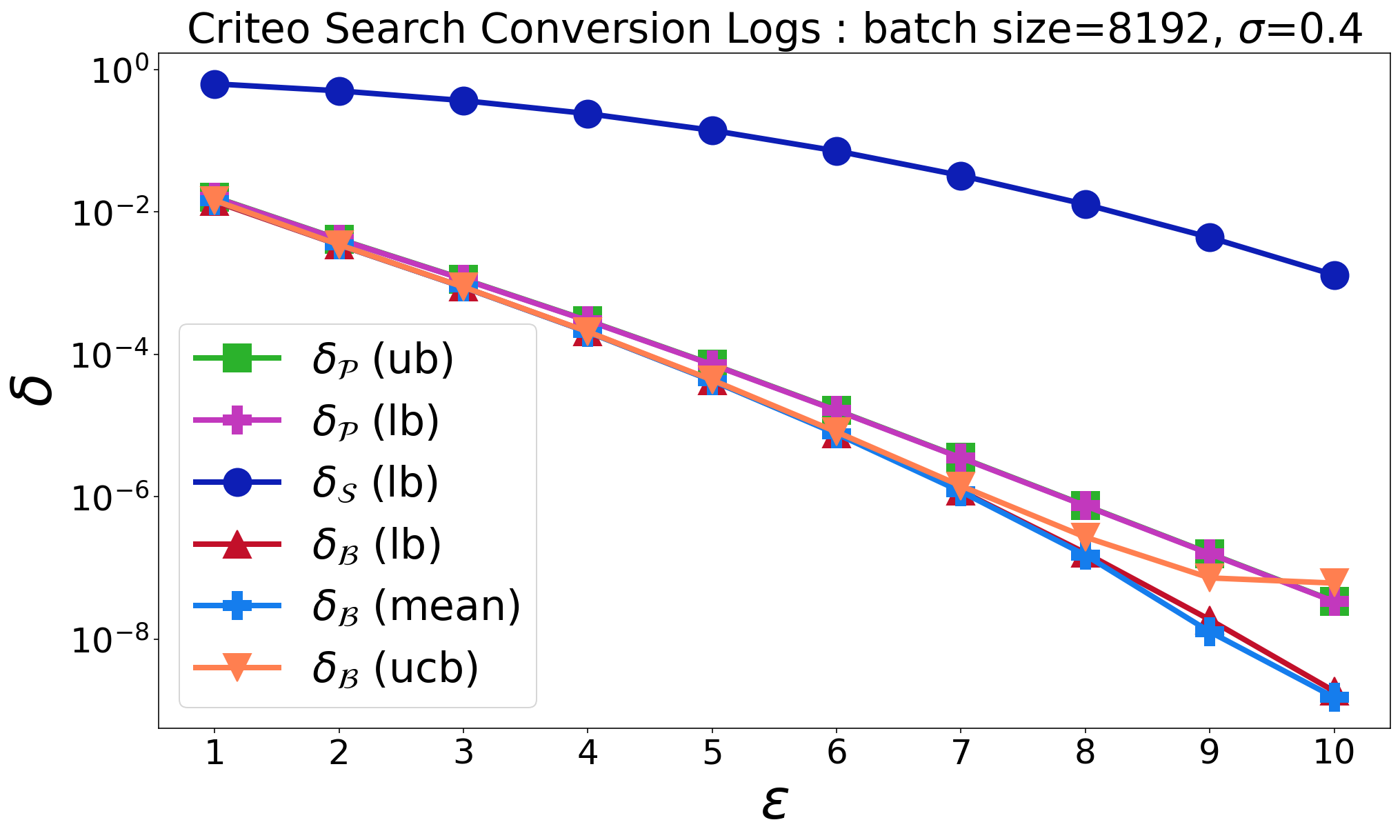} \\
\end{tabular}
\caption{Bounds on $\deltaP$, $\deltaS$,  and $\deltaB$ are plotted for various values of $\eps$ for different (expected) batch size and $\sigma$. These mean and upper confidence bounds for $\deltaB$ were obtained using order statistics sampling (specific orders and sample complexity specified in \Cref{app:training}).}
\label{fig:privacy}
\end{figure*}

\vspace{-2mm}
\section{DISCUSSION}\label{sec:discussion}
\vspace{-2mm}

We introduce the \BnB sampler for $\DPSGD$, and showed that it enjoys similar model utility as $\DPSGD$ with shuffling, and enjoys privacy amplification that is similar to Poisson subsampling in practical regimes. In order to do so efficiently, we developed the techniques of importance sampling and order statistics sampling. While in our paper we primarily considered a single epoch of training, our approaches also extend to multiple epochs as discussed in \Cref{subapp:multiple-epochs}.

Our work leaves several directions open for future investigation. The main open problem is to obtain a tight provable privacy accounting for $\ABLQB$, unlike the high probability bounds that we establish or to establish through some approximation that it is no worse than $\ABLQP$ in relevant regimes. An efficient method for tight accounting will also be useful to perform ``inverse'' accounting, namely to find $\sigma$ for a desired choice of $(\eps, \delta)$. Another alternative is to obtain R\'enyi DP guarantees~\citep{mironov17renyi}.

Subsequent to our work, \cite{feldmanshenfeld25} showed that the privacy guarantee of the balls-and-bins sampling is not worse than that of Poisson subsampling in a certain asymptotic sense. Furthermore, they also propose non-asymptotic bounds via decomposing the privacy loss distribution of Poisson subsampling, as well as via R\'enyi DP. These approaches are a promising avenue for overcoming the limitations of Monte Carlo sampling in this work.

\vspace{-2mm}
\subsubsection*{Acknowledgments}
\vspace{-2mm}
We thank anonymous reviewers for their thoughtful comments and suggestions that have improved the quality of this paper. We also thank the authors of \cite{choquettechoo24near} for helpful discussions regarding their concurrent work.

\newpage
\bibliographystyle{plainnat}
\bibliography{main.bbl}

\begin{thebibliography}{42}
\providecommand{\natexlab}[1]{#1}
\providecommand{\url}[1]{\texttt{#1}}
\expandafter\ifx\csname urlstyle\endcsname\relax
  \providecommand{\doi}[1]{doi: #1}\else
  \providecommand{\doi}{doi: \begingroup \urlstyle{rm}\Url}\fi

\bibitem[Abadi et~al.(2016)Abadi, Chu, Goodfellow, McMahan, Mironov, Talwar,
  and Zhang]{abadi16deep}
Mart{\'{\i}}n Abadi, Andy Chu, Ian~J. Goodfellow, H.~Brendan McMahan, Ilya
  Mironov, Kunal Talwar, and Li~Zhang.
\newblock Deep learning with differential privacy.
\newblock In \emph{CCS}, pages 308--318, 2016.

\bibitem[Anil et~al.(2022)Anil, Ghazi, Gupta, Kumar, and
  Manurangsi]{anil22dpbert}
Rohan Anil, Badih Ghazi, Vineet Gupta, Ravi Kumar, and Pasin Manurangsi.
\newblock Large-scale differentially private {BERT}.
\newblock In \emph{EMNLP (Findings)}, 2022.

\bibitem[Annamalai(2024)]{annamalai2024loss}
Meenatchi Sundaram Muthu~Selva Annamalai.
\newblock It's our loss: No privacy amplification for hidden state {DP-SGD}
  with non-convex loss.
\newblock In \emph{AISec}, pages 24--30, 2024.

\bibitem[Annamalai et~al.(2024)Annamalai, Balle, Cristofaro, and
  Hayes]{annamalai24shuffle}
Meenatchi Sundaram Muthu~Selva Annamalai, Borja Balle, Emiliano~De Cristofaro,
  and Jamie Hayes.
\newblock To shuffle or not to shuffle: Auditing {DP-SGD} with shuffling.
\newblock \emph{CoRR}, abs/2411.10614, 2024.

\bibitem[Balle and Wang(2018)]{balle18improving}
Borja Balle and Yu{-}Xiang Wang.
\newblock Improving the gaussian mechanism for differential privacy: Analytical
  calibration and optimal denoising.
\newblock In \emph{ICML}, 2018.

\bibitem[Balle et~al.(2020)Balle, Kairouz, McMahan, Thakkar, and
  Guha~Thakurta]{balle20checkin}
Borja Balle, Peter Kairouz, Brendan McMahan, Om~Thakkar, and Abhradeep
  Guha~Thakurta.
\newblock Privacy amplification via random check-ins.
\newblock In \emph{NeurIPS}, pages 4623--4634, 2020.

\bibitem[Balle et~al.(2022)Balle, Berrada, De, Ghalebikesabi, Hayes, Pappu,
  Smith, and Stanforth]{jax-privacy2022github}
Borja Balle, Leonard Berrada, Soham De, Sahra Ghalebikesabi, Jamie Hayes,
  Aneesh Pappu, Samuel~L Smith, and Robert Stanforth.
\newblock {JAX}-{P}rivacy: Algorithms for privacy-preserving machine learning
  in {JAX}, 2022.
\newblock URL \url{http://github.com/google-deepmind/jax_privacy}.

\bibitem[Beltran et~al.(2024)Beltran, Tobaben, J\"alk\"o, Loppi, and
  Honkela]{beltran24towards}
Sebastian~Rodriguez Beltran, Marlon Tobaben, Joonas J\"alk\"o, Niki Loppi, and
  Antti Honkela.
\newblock Towards efficient and scalable training of differentially private
  deep learning.
\newblock In \emph{NeurIPS}, 2024.

\bibitem[Ben~Slimane(2001)]{slimane01bounds}
S.~Ben~Slimane.
\newblock Bounds on the distribution of a sum of independent lognormal random
  variables.
\newblock \emph{IEEE Trans. Comm.}, 49\penalty0 (6):\penalty0 975--978, 2001.

\bibitem[Bradbury et~al.(2018)Bradbury, Frostig, Hawkins, Johnson, Leary,
  Maclaurin, Necula, Paszke, Vander{P}las, Wanderman-{M}ilne, and
  Zhang]{jax2018github}
James Bradbury, Roy Frostig, Peter Hawkins, Matthew~James Johnson, Chris Leary,
  Dougal Maclaurin, George Necula, Adam Paszke, Jake Vander{P}las, Skye
  Wanderman-{M}ilne, and Qiao Zhang.
\newblock {JAX}: composable transformations of {P}ython+{N}um{P}y programs,
  2018.
\newblock URL \url{http://github.com/jax-ml/jax}.

\bibitem[Choquette{-}Choo et~al.(2025)Choquette{-}Choo, Ganesh, Haque, Steinke,
  and Thakurta]{choquettechoo24near}
Christopher~A. Choquette{-}Choo, Arun Ganesh, Saminul Haque, Thomas Steinke,
  and Abhradeep Thakurta.
\newblock Near exact privacy amplification for matrix mechanisms.
\newblock In \emph{ICLR}, 2025.

\bibitem[Chua et~al.(2024{\natexlab{a}})Chua, Ghazi, Kamath, Kumar, Manurangsi,
  Sinha, and Zhang]{chua24private}
Lynn Chua, Badih Ghazi, Pritish Kamath, Ravi Kumar, Pasin Manurangsi, Amer
  Sinha, and Chiyuan Zhang.
\newblock How private are {DP-SGD} implementations?
\newblock In \emph{ICML}, 2024{\natexlab{a}}.

\bibitem[Chua et~al.(2024{\natexlab{b}})Chua, Ghazi, Kamath, Kumar, Manurangsi,
  Sinha, and Zhang]{chua24scalable}
Lynn Chua, Badih Ghazi, Pritish Kamath, Ravi Kumar, Pasin Manurangsi, Amer
  Sinha, and Chiyuan Zhang.
\newblock Scalable {DP-SGD}: Shuffling vs. {P}oisson subsampling.
\newblock In \emph{NeurIPS}, 2024{\natexlab{b}}.

\bibitem[David and Nagaraja(2004)]{david04order}
Herbert~A David and Haikady~N Nagaraja.
\newblock \emph{Order Statistics}.
\newblock John Wiley \& Sons, 2004.

\bibitem[De et~al.(2022)De, Berrada, Hayes, Smith, and Balle]{de22unlocking}
Soham De, Leonard Berrada, Jamie Hayes, Samuel~L. Smith, and Borja Balle.
\newblock Unlocking high-accuracy differentially private image classification
  through scale.
\newblock \emph{CoRR}, abs/2204.13650, 2022.

\bibitem[Dean and Ghemawat(2004)]{dean04mapreduce}
Jeffrey Dean and Sanjay Ghemawat.
\newblock Mapreduce: Simplified data processing on large clusters.
\newblock In \emph{OSDI}, pages 137--150, 2004.

\bibitem[Desfontaines(2021)]{desfontainesblog2021list}
Damien Desfontaines.
\newblock A list of real-world uses of differential privacy.
\newblock
  \url{https://desfontain.es/blog/real-world-differential-privacy.html}, Oct
  2021.
\newblock Ted is writing things (personal blog).

\bibitem[Dockhorn et~al.(2023)Dockhorn, Cao, Vahdat, and
  Kreis]{dockhorn2022differentially}
Tim Dockhorn, Tianshi Cao, Arash Vahdat, and Karsten Kreis.
\newblock Differentially private diffusion models.
\newblock \emph{TMLR}, 2023.

\bibitem[Dwork and Rothblum(2016)]{dwork16concentrated}
Cynthia Dwork and Guy~N. Rothblum.
\newblock Concentrated differential privacy.
\newblock \emph{CoRR}, abs/1603.01887, 2016.

\bibitem[Feldman and Shenfeld(2025)]{feldmanshenfeld25}
Vitaly Feldman and Moshe Shenfeld.
\newblock Privacy amplification by random allocation.
\newblock \emph{CoRR}, abs/2502.08202, 2025.

\bibitem[{Google's DP Library.}(2020)]{GoogleDP}
{Google's DP Library.}
\newblock D{P} {A}ccounting {L}ibrary, 2020.
\newblock URL
  \url{https://github.com/google/differential-privacy/tree/main/python/dp_accounting}.

\bibitem[He et~al.(2023)He, Li, Yu, Zhang, Kulkarni, Lee, Backurs, Yu, and
  Bian]{he2022exploring}
Jiyan He, Xuechen Li, Da~Yu, Huishuai Zhang, Janardhan Kulkarni, Yin~Tat Lee,
  Arturs Backurs, Nenghai Yu, and Jiang Bian.
\newblock Exploring the limits of differentially private deep learning with
  group-wise clipping.
\newblock In \emph{ICLR}, 2023.

\bibitem[Hoeffding(1963)]{hoeffding63probability}
Wassily Hoeffding.
\newblock Probability inequalities for sums of bounded random variables.
\newblock \emph{J. ASA}, 58\penalty0 (301):\penalty0 13--30, 1963.

\bibitem[Igamberdiev et~al.(2024)Igamberdiev, Vu, K{\"u}nnecke, Yu, Holmer, and
  Habernal]{igamberdiev2023dp}
Timour Igamberdiev, Doan Nam~Long Vu, Felix K{\"u}nnecke, Zhuo Yu, Jannik
  Holmer, and Ivan Habernal.
\newblock {DP-NMT}: Scalable differentially-private machine translation.
\newblock In \emph{EACL (Demonstrations)}, pages 94--105, 2024.

\bibitem[Jean-Baptiste~Tien(2014)]{tien14criteokaggle}
Olivier~Chapelle Jean-Baptiste~Tien, joycenv.
\newblock Display advertising challenge, 2014.
\newblock URL
  \url{https://kaggle.com/competitions/criteo-display-ad-challenge}.

\bibitem[Kairouz et~al.(2015)Kairouz, Oh, and Viswanath]{kairouz15composition}
Peter Kairouz, Sewoong Oh, and Pramod Viswanath.
\newblock The composition theorem for differential privacy.
\newblock In \emph{ICML}, pages 1376--1385, 2015.

\bibitem[Kairouz et~al.(2021)Kairouz, McMahan, Song, Thakkar, Thakurta, and
  Xu]{kairouz21practical}
Peter Kairouz, Brendan McMahan, Shuang Song, Om~Thakkar, Abhradeep Thakurta,
  and Zheng Xu.
\newblock Practical and private (deep) learning without sampling or shuffling.
\newblock In \emph{ICML}, pages 5213--5225, 2021.

\bibitem[Koskela et~al.(2020)Koskela, J{\"a}lk{\"o}, and
  Honkela]{koskela2020computing}
Antti Koskela, Joonas J{\"a}lk{\"o}, and Antti Honkela.
\newblock Computing tight differential privacy guarantees using {FFT}.
\newblock In \emph{AISTATS}, pages 2560--2569, 2020.

\bibitem[Koskela et~al.(2023)Koskela, Heikkil{\"{a}}, and
  Honkela]{koskela23numerical}
Antti Koskela, Mikko~A. Heikkil{\"{a}}, and Antti Honkela.
\newblock Numerical accounting in the shuffle model of differential privacy.
\newblock \emph{TMLR}, 2023, 2023.

\bibitem[Lebeda et~al.(2024)Lebeda, Regehr, and Kamath]{lebeda2024avoiding}
Christian~Janos Lebeda, Matthew Regehr, and Gautam Kamath.
\newblock Avoiding pitfalls for privacy accounting of subsampled mechanisms
  under composition.
\newblock \emph{CoRR}, abs/2405.20769, 2024.

\bibitem[Marsaglia and Tsang(2000)]{marsaglia00simple}
George Marsaglia and Wai~Wan Tsang.
\newblock A simple method for generating gamma variables.
\newblock \emph{ACM Trans. Math. Softw.}, 26\penalty0 (3):\penalty0 363–372,
  2000.

\bibitem[McMahan et~al.(2022)McMahan, Rush, and Thakurta]{mcmahan22dpmf}
Brendan McMahan, Keith Rush, and Abhradeep~Guha Thakurta.
\newblock Private online prefix sums via optimal matrix factorizations.
\newblock \emph{CoRR}, abs/2202.08312, 2022.

\bibitem[Microsoft.(2021)]{MicrosoftDP}
Microsoft.
\newblock A fast algorithm to optimally compose privacy guarantees of
  differentially private ({DP}) mechanisms to arbitrary accuracy., 2021.
\newblock URL \url{https://github.com/microsoft/prv_accountant}.

\bibitem[Mironov(2017)]{mironov17renyi}
Ilya Mironov.
\newblock R{\'{e}}nyi differential privacy.
\newblock In \emph{CSF}, pages 263--275, 2017.

\bibitem[Ponomareva et~al.(2023)Ponomareva, Hazimeh, Kurakin, Xu, Denison,
  McMahan, Vassilvitskii, Chien, and Thakurta]{ponomareva23dpfy}
Natalia Ponomareva, Hussein Hazimeh, Alex Kurakin, Zheng Xu, Carson Denison,
  H.~Brendan McMahan, Sergei Vassilvitskii, Steve Chien, and Abhradeep~Guha
  Thakurta.
\newblock How to dp-fy {ML:} {A} practical guide to machine learning with
  differential privacy.
\newblock \emph{J. AIR}, 77:\penalty0 1113--1201, 2023.

\bibitem[Prediger and Koskela(2020)]{DPBayes}
Lukas Prediger and Antti Koskela.
\newblock Code for computing tight guarantees for differential privacy., 2020.
\newblock URL \url{https://github.com/DPBayes/PLD-Accountant}.

\bibitem[Tallis and Yadav(2018)]{tallis2018reacting}
Marcelo Tallis and Pranjul Yadav.
\newblock Reacting to variations in product demand: An application for
  conversion rate {(CR)} prediction in sponsored search.
\newblock \emph{CoRR}, abs/1806.08211, 2018.

\bibitem[Tang et~al.(2024)Tang, Panda, Nasr, Mahloujifar, and
  Mittal]{tang2024private}
Xinyu Tang, Ashwinee Panda, Milad Nasr, Saeed Mahloujifar, and Prateek Mittal.
\newblock Private fine-tuning of large language models with zeroth-order
  optimization.
\newblock \emph{CoRR}, abs/2401.04343, 2024.

\bibitem[Tensorflow Privacy()]{tf_privacy}
Tensorflow Privacy, 2024.
\newblock URL
  \url{https://www.tensorflow.org/responsible_ai/privacy/api_docs/python/tf_privacy}.

\bibitem[Wang et~al.(2023)Wang, Mahloujifar, Wu, Jia, and
  Mittal]{wang23randomized}
Jiachen~(Tianhao) Wang, Saeed Mahloujifar, Tong Wu, Ruoxi Jia, and Prateek
  Mittal.
\newblock A randomized approach to tight privacy accounting.
\newblock In \emph{NeurIPS}, pages 33856--33893, 2023.

\bibitem[Yousefpour et~al.(2021)Yousefpour, Shilov, Sablayrolles, Testuggine,
  Prasad, Malek, Nguyen, Ghosh, Bharadwaj, Zhao, Cormode, and
  Mironov]{yousefpour21opacus}
Ashkan Yousefpour, Igor Shilov, Alexandre Sablayrolles, Davide Testuggine,
  Karthik Prasad, Mani Malek, John Nguyen, Sayan Ghosh, Akash Bharadwaj,
  Jessica Zhao, Graham Cormode, and Ilya Mironov.
\newblock Opacus: User-friendly differential privacy library in {PyTorch}.
\newblock \emph{CoRR}, abs/2109.12298, 2021.

\bibitem[Zhu et~al.(2022)Zhu, Dong, and Wang]{zhu22optimal}
Yuqing Zhu, Jinshuo Dong, and Yu{-}Xiang Wang.
\newblock Optimal accounting of differential privacy via characteristic
  function.
\newblock In \emph{AISTATS}, pages 4782--4817, 2022.

\end{thebibliography}

\newpage
\section*{Checklist}

\newcommand{\YES}{{\bf Yes}}
\newcommand{\NO}{{\bf No}}
\newcommand{\NA}{{\bf Not Applicable}}

\begin{enumerate}
\item For all models and algorithms presented, check if you include:
\begin{enumerate}
\item A clear description of the mathematical setting, assumptions, algorithm, and/or model. \YES
\item An analysis of the properties and complexity (time, space, sample size) of any algorithm. \YES
\item (Optional) Anonymized source code, with specification of all dependencies, including external libraries. \YES : An implementation of the privacy accounting algorithms introduced in this work are available at \href{https://github.com/google-research/google-research/tree/master/dpsgd_batch_sampler_accounting}{\small github.com/google-research/google-research/ tree/master/dpsgd\_batch\_sampler\_accounting}.
\end{enumerate}

\item For any theoretical claim, check if you include:
\begin{enumerate}
\item Statements of the full set of assumptions of all theoretical results. \YES
\item Complete proofs of all theoretical results. \YES
\item Clear explanations of any assumptions. \YES
\end{enumerate}

\item For all figures and tables that present empirical results, check if you include:
\begin{enumerate}
\item The code, data, and instructions needed to reproduce the main experimental results (either in the supplemental material or as a URL). \YES : An implementation of the privacy accounting algorithms introduced in this work are available at \href{https://github.com/google-research/google-research/tree/master/dpsgd_batch_sampler_accounting}{\small github.com/google-research/google-research/ tree/master/dpsgd\_batch\_sampler\_accounting}. We do not include the code needed to evaluate our experiments with DP-SGD on neural network architectures, as they are orthogonal to the main contributions of this work.
\item All the training details (e.g., data splits, hyperparameters, how they were chosen). \YES : Experimental details are presented in \Cref{app:training}.
 \item A clear definition of the specific measure or statistics and error bars (e.g., with respect to the random seed after running experiments multiple times). \YES
 \item A description of the computing infrastructure used. (e.g., type of GPUs, internal cluster, or cloud provider). \YES: We include details in \Cref{app:training}.
\end{enumerate}

\item If you are using existing assets (e.g., code, data, models) or curating/releasing new assets, check if you include:
\begin{enumerate}
\item Citations of the creator If your work uses existing assets. \NA
\item The license information of the assets, if applicable. \NA
\item New assets either in the supplemental material or as a URL, if applicable. \NA
\item Information about consent from data providers/curators. \NA
\item Discussion of sensible content if applicable, e.g., personally identifiable information or offensive content. \NA
\end{enumerate}

\item If you used crowdsourcing or conducted research with human subjects, check if you include:
\begin{enumerate}
\item The full text of instructions given to participants and screenshots. \NA
\item Descriptions of potential participant risks, with links to Institutional Review Board (IRB) approvals if applicable. \NA
\item The estimated hourly wage paid to participants and the total amount spent on participant compensation. \NA
\end{enumerate}
\end{enumerate}

\newpage

\onecolumn
\thispagestyle{plain}\thispagestyle{empty}
\aistatstitle{Balls-and-Bins Sampling for DP-SGD: Supplementary Material}

\appendix

\section{Batch Generators}\label{app:batch-gen}

We formally describe the Deterministic ($\cD$), Shuffle ($\cS$), and Poisson ($\cP$) batch generators considered in this work, as formalized in \cite{chua24private}. Let $n$ be the number of datapoints.

\begin{itemize}%
	\item Deterministic $\cD_{b, T}$, formalized in \Cref{alg:deterministic-batch}, generates $T$ batches each of size $b$ in the given sequential order of the dataset. This method requires that $n = b \cdot T$.
	\item Shuffle $\cS_{b,T}$, formalized in \Cref{alg:shuffle-batch}, is similar to $\cD_{b, T}$, but first applies a random permutation to the dataset. This method also requires $n = b \cdot T$.
	\item Poisson $\cP_{b,T}$, formalized in \Cref{alg:poisson-batch}, samples each batch independently by including each example with probability $\frac bn$. This method works for any $n$ and results in an expected batch size of $b$.
\end{itemize}

{
\newcommand{\algtextfont}{\fontsize{8.5pt}{10pt}\selectfont}
\newcommand{\highlight}[1]{\colorbox{blue!10}{\textcolor{black}{#1}}}

\begin{figure}[H]
\begin{minipage}[H]{0.48\textwidth}
\begin{algorithm}[H]
\caption{$\cD_{b,T}$: Deterministic Batch Generator}
\label{alg:deterministic-batch}
\begin{algorithmic}
\PARAMETERS Batch size $b$, number of batches $T$.
\REQUIRE Number of datapoints $n = b \cdot T$.
\ENSURE Seq. of disjoint batches $S_1, \ldots, S_T \subseteq [n]$.
\FOR{$t = 0, \ldots, T-1$}
    \STATE $S_{t+1} \gets \{tb + 1, \ldots, tb + b\}$
\ENDFOR
\RETURN $S_1, \ldots, S_T$
\end{algorithmic}
\end{algorithm}
\end{minipage}
\hspace*{3pt}
{\vrule width 0.5pt}
\hspace*{3pt}
\begin{minipage}[H]{0.48\textwidth}
\begin{algorithm}[H]
\caption{$\cS_{b,T}$: Shuffle Batch Generator}
\label{alg:shuffle-batch}
\begin{algorithmic}
\PARAMETERS Batch size $b$, number of batches $T$.
\REQUIRE Number of datapoints $n = b \cdot T$.
\ENSURE Seq. of disjoint batches $S_1, \ldots, S_T \subseteq [n]$.
\STATE $\pi \gets$ random permutation over $[n]$
\FOR{$t = 0, \ldots, T-1$}
    \STATE $S_{t+1} \gets \{\pi(tb + 1), \ldots, \pi(tb + b)\}$
\ENDFOR
\RETURN $S_1, \ldots, S_T$
\end{algorithmic}
\end{algorithm}
\end{minipage}
\vspace*{-1mm}
\end{figure}

\begin{algorithm}[H]
\caption{$\cP_{b,T}$: Poisson Batch Generator}
\label{alg:poisson-batch}
\begin{algorithmic}
\PARAMETERS Expected batch size $b$, num. of batches $T$.
\REQUIRE Number of datapoints $n$.
\ENSURE Seq. of batches $S_1, \ldots, S_T \subseteq [n]$.
\FOR{$t = 1, \ldots, T$}
    \STATE $S_{t} \gets \emptyset$
    \FOR{$i = 1, \ldots, n$}
        \STATE $S_{t} \gets \begin{cases}
            S_{t} \cup \{i\} & \text{ with probability } b/n\\
            S_{t} & \text{ with probability } 1 - b/n\\
        \end{cases}$
    \ENDFOR
\ENDFOR
\RETURN $S_1, \ldots, S_T$
\end{algorithmic}
\end{algorithm}

}

We recall that since we are using the ``zeroing-out'' adjacency, $n$ is known, and not protected under DP.

\section{Importance and Order Statistics Sampling}\label{app:sampling}

We describe how to efficiently perform importance sampling as described in \Cref{alg:importance-mce} for the pair $(\PB, \QB)$ as well as the proof that \Cref{alg:order-sampling} samples from the joint distribution of order statistics. In order to do so, we use the connection between the Beta distribution and order statistics~\citep[see, e.g.,][]{david04order}.

First, we establish some notation that we use throughout this section. Let $\Unif[a, b]$ denote the uniform distribution over the interval $[a, b]$. For any distribution $P$ over $\R$, let $\CDF_P(x) := \Pr_{z \sim P}[z \le x]$ denote the cumulative density function, and let $\CDF_P^{-1}(\cdot)$ denote its inverse.\footnote{In cases where $\CDF_P(\cdot)$ is not a continuous function, the inverse is defined as $\CDF_P^{-1}(y) := \min_{x \in \R : \CDF_P(x) \ge y} x$; the minimum always exists since $\CDF_P$ is right continuous. However, since we only deal with distributions with continuous $\CDF$s, this detail is not going to be important.}
For any event (measurable set) $E$, let $P|_{E}$ denote the distribution of $P$ conditioned on event $E$. In this work, we only use distributions with probability measures that are continuous with respect to the Lebesgue measure. Even though the following techniques extend to the non-continuous distributions, we assume that distributions are continuous below.

\begin{definition}[Beta Distribution]\label{def:beta-dist}
The $\Beta(\alpha, \beta)$ distribution over $[0, 1]$ is defined by the density function
\[
f(x; \alpha, \beta) := \frac{\Gamma(\alpha+\beta)}{\Gamma(\alpha) \Gamma(\beta)} x^{\alpha-1} (1 - x)^{\beta-1}.
\]
\end{definition}

\begin{fact}[Order Statistics and Beta Distribution]\label{fact:order-beta}
The random variable $y^{(k)}$ that is the $k$th largest element among $x_1, \ldots, x_R \sim \Unif[0, 1]$ is distributed as $\Beta(R-k+1, k)$.
\end{fact}

An important primitive we use in our sampling methods is the ability to efficiently sample from $\Beta(\alpha, \beta)$ distributions~\citep[see, e.g.,][]{marsaglia00simple}, with efficient implementations available, for example in Python, using the class \texttt{scipy.stats.beta}.

\begin{fact}[Probability Integral Transform]\label{fact:prob-int-transform}
Let $P$ be any distribution over $\R$.
For $x \sim P$, $\CDF_P(x)$ is distributed as $\Unif[0, 1]$. Conversely, for $y \sim \Unif[0, 1]$, $\CDF_P^{-1}(y)$ is distributed as $P$.%

Furthermore it follows that, for any interval $[a, b] \in \R$, the distribution of $\CDF_P(x)$ for $x \sim P|_{[a, b]}$ is $\Unif[\CDF_P(a), \CDF_P(b)]$, and conversely for $y \sim \Unif[\CDF_P(a), \CDF_P(b)]$, $\CDF_P^{-1}(y)$ is distributed as $P|_{[a, b]}$.
\end{fact}

Thus, \Cref{fact:prob-int-transform} implies that for any distribution $P$ over $\R$ for which $\CDF_P$ and $\CDF_P^{-1}$ are efficiently computable, it is possible to sample from $P$ conditioned on the sample being in any specified range $[a, b]$.
Since $\CDF_{\Beta(\alpha, \beta)}$, $\CDF_{\cN(0, \sigma^2)}$ and their inverses are efficiently computable, for example in Python using the classes \texttt{scipy.stats.beta} and \texttt{scipy.stats.norm} respectively, we can sample from the conditional $\Beta(\alpha, \beta)$ and $\cN(0, \sigma^2)$ distributions.

\Cref{fact:order-beta} and \ref{fact:prob-int-transform} together suggest the following approach to sample a single order statistics for sampling $R$ i.i.d. samples from $P$ or $P|_{[a, b]}$.
\begin{proposition}\label{prop:sample-single-order}
Let $P$ be any distribution over $\R$.
The random variable $y^{(k)}$ that is the $k$th largest element among $x_1, \ldots, x_R \sim P$, $\CDF_P(y^{(k)})$ is distributed as $\Beta(R-k+1, k)$. Conversely, for $z \sim \Beta(R-k+1, k)$, $\CDF_P^{-1}(z)$ has the same distribution as $y^{(k)}$.%

Furthermore it follows that, for any interval $[a, b] \in \R$, the distribution of $\CDF_P(y^{(k)})$ for $y^{(k)}$ being the $k$th largest element among $x_1, \ldots, x_R \sim P|_{[a, b]}$ is distributed as $\CDF_P(a) + (\CDF_P(b) - \CDF_P(a)) \cdot z$ for $z \sim \Beta(R-k+1, k)$, and conversely for $z \sim \Beta(R-k+1, k)$, $\CDF_P^{-1}(\CDF_P(a) + (\CDF_P(b) - \CDF_P(a)) \cdot z)$ has the same distribution as $y^{(k)}$.
\end{proposition}

\subsection{Efficient Importance Sampling}\label{subapp:importance}

In this section we describe how to efficiently estimate $\Deps(\QB ~\|~ \PB)$ and $\Deps(\PB ~\|~ \QB)$ using \Cref{alg:importance-mce}. We use $\Phi_\sigma(\cdot)$ to denote $\CDF_{\cN(0, \sigma^2)}$ for short.

\paragraph{\boldmath Estimating $\Deps(\QB ~\|~ \PB)$.} Recall that in this case, we wish to estimate $\Ex_{x \sim \QB |_{E_\eps}} \max\{0, 1 - e^{\eps - L_{\QB ~\|~ \PB}(x)}\}$ where $\QB = \cN(0, \sigma^2 I)$ and $E_\eps := \{x :\in \R^T : \max_{t \in [T]} x_t \le C_{\eps}\}$ for $C_\eps := \frac12 - \eps \sigma^2$. In order to sample from $\QB|_{E_\eps}$, we observe that this is equivalent to sampling $T$ coordinates i.i.d. from $\cN(0, \sigma^2)|_{\{x \,:\, x \le C_{\eps}\}}$. This can be done using \Cref{fact:prob-int-transform}, by sampling $y_t \sim \Unif[0, \Phi_\sigma(C_\eps)]$ and returning $x_t = \Phi_\sigma^{-1}(y_t)$ for each $t \in [T]$.

\paragraph{\boldmath Estimating $\Deps(\PB ~\|~ \QB)$.} Recall that in this case, we wish to estimate $\Ex_{x \sim P_1 |_{E_\eps}} \max\{0, 1 - e^{\eps - L_{\PB ~\|~ \QB}(x)}\}$ where $P_1 = \cN(e_1, \sigma^2 I)$ and $E_\eps := \{ x : \max\{ x_1 - 1, \max_{t > 1} x_t\} \ge C_\eps \}$ for
$C_{\eps} ~=~ \frac{1}{2} + \sigma^2 \cdot \prn{\eps - \log\prn{1 + \frac{e^{1/\sigma^2} - 1}{T}}}$.
The choice of $E_{\eps}$ is such that for $x \sim P_1 |_{E_\eps}$, the distribution of $x - e_1$ is the same as $\cN(0, \sigma^2 I_T)|_{\{x\,:\,\max_t x_t \ge C_\eps\}}$.
In \Cref{alg:max-of-samples-from-P}, we provide a generic algorithm that for any distribution $P$ over $\R$, samples from the distribution $P^{\otimes T}|_{\{x\, :\, \max_t x_t \ge C\}}$, i.e.,, samples from $T$ i.i.d. samples from $P$ conditioned on the maximum value being at least $C$.
Thus, we can sample from $P_1|_{E_\eps}$ by sampling $x' \sim \cN(0, \sigma^2 I_T)|_{\{x' \,:\,\max_t x'_t \ge C_{\eps}\}}$ using \Cref{alg:max-of-samples-from-P}, and returning $x = x' + e_1$.

\begin{algorithm}[t]
\caption{Sampling from $P^{\otimes T}$, conditioned on the maximum value being at least $C$}
\label{alg:max-of-samples-from-P}
\begin{algorithmic}
\REQUIRE Distribution $P$ over $\R$, lower bound $C \in \R$ on the maximum value.
\ENSURE Sample $x \sim P|_{\max_t x_t \ge C}$
\STATE $y_* \sim \Beta(T, 1)|_{[\CDF_P(C), 1]}$ (using \Cref{fact:prob-int-transform})
\STATE $t_* \sim$ uniformly random coordinate in $[T]$
\FOR{$t \in \{1, \ldots, T\}$}
    \IF{$t = t_*$}
        \STATE $z_t \gets y_*$
    \ELSE
        \STATE $z_t \sim \Unif[0, y_*]$
    \ENDIF
\ENDFOR
\RETURN $(\CDF_P^{-1}(z_t) : t \in [T])$
\end{algorithmic}
\end{algorithm}

\paragraph{Numerical Evaluation.}
To demonstrate the usefulness of our importance sampling method, in \Cref{fig:importance-advantage}, we plot the upper confidence bound on $\deltaB(\eps)$ as obtained via \Cref{alg:mce} (i.e., without importance sampling) and via \Cref{alg:importance-mce} (i.e., with importance sampling) along the lower bound obtained via \eqref{eq:bnb-lb}. The upper confidence bounds are obtained for error probability $\beta = 10^{-3}$. For a similar running time, we see that \Cref{alg:importance-mce} is able to get significantly tighter upper confidence bounds in each setting. This is made possible because the importance sampling is able to ``zoom in'' into events of tiny probability. For example, in the left part of \Cref{fig:importance-advantage} for $T = 5000$ and $\sigma = 0.4$, at $\eps = 12$, the importance sampler using $m = 200,000$ samples is considering an event $E_\eps$ such that $\PB(E_{\eps}) \approx 3.75 \cdot 10^{-3}$, and on the right for $T = 10\,000$ and $\sigma = 0.35$, at $\eps = 12$, the importance sampler using $m = 100,000$ samples is considering an event $E_\eps$ such that $\PB(E_\eps) \approx 1.66 \cdot 10^{-4}$.
Recall that the reduction in sample complexity due to our use of importance sampling is by a factor of $1 / \PB(E_\eps)$.

\def\figheight{0.27}
\begin{figure*}
\centering
\begin{tabular}{cc}
\includegraphics[height=\figheight\linewidth]{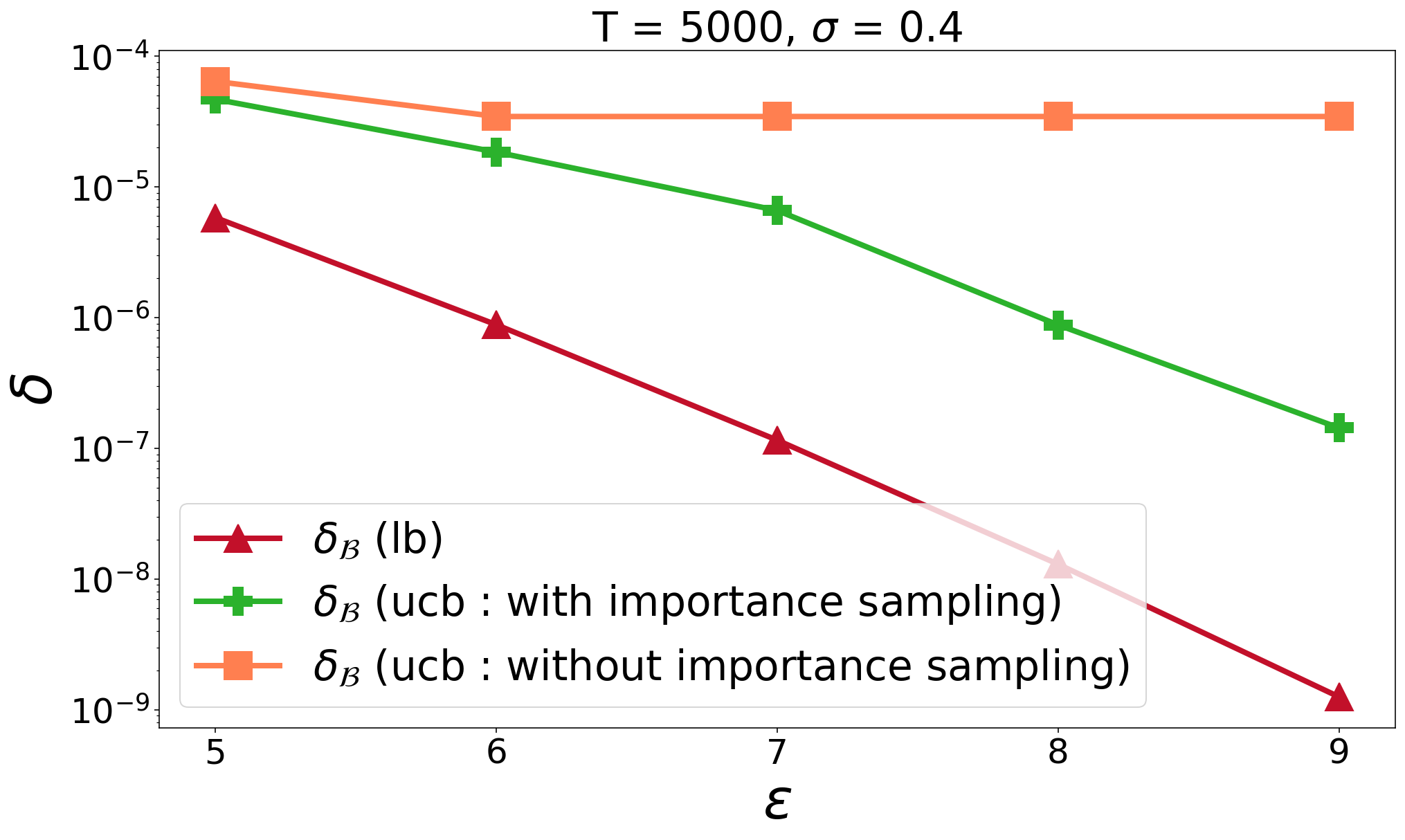} & 
\includegraphics[height=\figheight\linewidth]{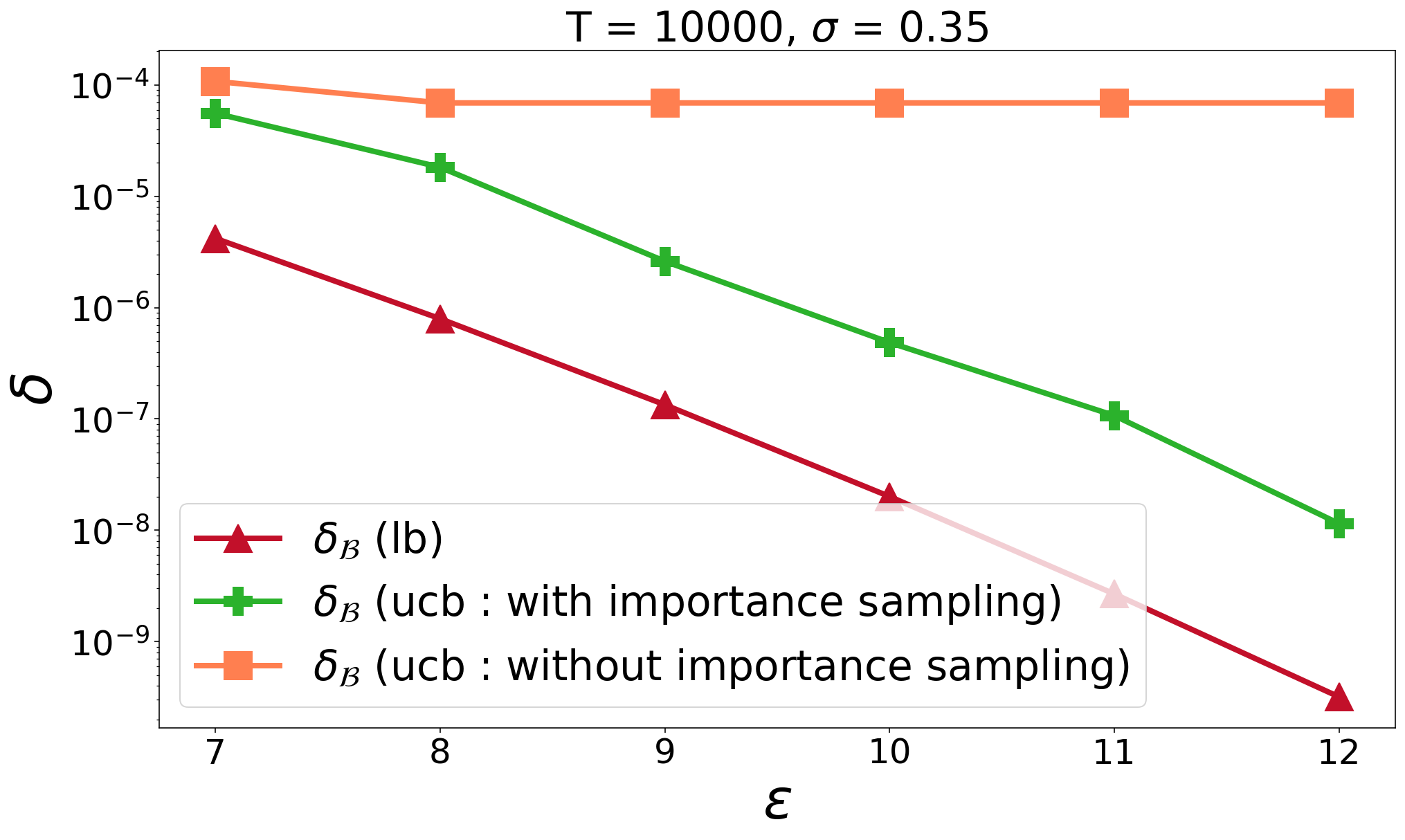}
\end{tabular}
\caption{Upper confidence bounds on $\deltaB(\eps)$ against various values of $\eps$ for two settings of $T$ and $\sigma$, with and without importance sampling. Additionally, lower bounds on $\deltaB(\eps)$ are included.}
\label{fig:importance-advantage}
\end{figure*}

\subsection{Order Statistics Sampling}\label{subapp:order-stats}

We show that \Cref{alg:order-sampling} indeed samples from the joint distribution of order statistics of $P$.
\begin{theorem}\label{thm:order-sampling}
For any distribution $P$ over $\R$, and number of random variables $R$ and order statistic indices $k_1, \ldots, k_r$, the values $(y^{(k_1)}, \ldots, y^{(k_r)})$ returned by \Cref{alg:order-sampling} are distributed as the $k_1, \ldots, k_r$ largest elements among $x_1, \ldots, x_R \sim P$.
\end{theorem}
\begin{proof}
We prove the statement via induction on $r$. When $r = 1$, we have from \Cref{fact:order-beta}, that $\CDF_P^{-1}(z_1)$ for $z_1 \sim \Beta(R-k_1+1, k_1)$ has the same distribution as the $k_1$th order statistic.

For $r > 1$, suppose we inductively assume that $(y^{(k_1)}, \ldots, y^{(k_{r-1})})$ are jointly distributed as the $(k_1, \ldots, k_{r-1})$ order statistics. Note that $\CDF_P(y^{(k_i)}) = \prod_{j=1}^i z_j$ for all $i$. The conditional distribution of $y^{(k_r)}$ given $(y^{(k_1)}, \ldots, y^{(k_{r-1})})$ is the same as the $(k_r - k_{r-1})$th order statistic among $R - k_r$ random variables drawn from $P_{(-\infty, y^{(k_{r-1})}]}$. Using \Cref{prop:sample-single-order}, we have that for $z_r \sim \Beta(R - k_{r-1} + 1, k_r - k_{r-1})$, $\CDF_P^{-1}(\CDF_P(y^{(k_{r-1})}) \cdot z_r))$ is distributed as per this conditional distribution. Since $\CDF_P(y^{(k_{r-1})}) = \prod_{j=1}^{r-1} z_j$ the induction argument is complete.
\end{proof}

\paragraph{Numerical evaluation.}
To demonstrate the usefulness of our order statistics sampling method, in \Cref{fig:order-stats-advantage}, we plot the upper confidence bound on $\deltaB(\eps)$ as obtained via \Cref{alg:mce} as is (i.e., without order statistics sampling) and with order statistics sampling \Cref{alg:order-sampling} (i.e., with an upper bound on the loss function) along the lower bound obtained via \eqref{eq:bnb-lb}. The sub-figures in \Cref{fig:order-stats-advantage} were generated as follows.
\begin{itemize}
\item The figure on the left used $\sigma = 0.32$ and number of steps $T = 100,000$. The estimates without order statistics used $m=10,000$ samples, whereas, the estimates with order statistics used $m=100,000$ samples, using the order statistics of $(1, 2, \ldots, 400, 410, \ldots, 1000, 1100, \ldots, 10\,000, 11\,000, \ldots, 50\,000)$ (a total of $590$ orders) were used. Despite using $10$ times more samples, the estimation with order statistics ran in $\sim 66$ seconds, which is $\approx 25\%$ of the time needed without order statistics sampling ($\sim 268$ seconds).
	
\item The figure on the right used $\sigma = 0.25$ and number of steps $T = 1,000,000$. The estimates without order statistics used $m=1000$ samples, whereas, the estimates with order statistics used $m=300,000$ samples, using the order statistics of $(1, 2, \ldots, 300, 310, \ldots, 1000, 1100, \ldots, 10\,000, 11\,000, \ldots, 100\,000, 110\,000, \ldots, 500\,000)$ (a total of $590$ orders) were used. Despite using $300$ times more samples, the estimation with order statistics ran in $\sim 203$ seconds, which is $\approx 82\%$ of the time needed without order statistics sampling ($\sim 245$ seconds).
\end{itemize}

Running times can vary significantly depending on the machine; these figures are offered as a rough guide only.  The running time scales linearly with sample size and the number of order statistics and thus, these times are indicative of performance with more samples or varied order statistics.

\def\figheight{0.27}
\begin{figure*}
\centering
\begin{tabular}{cc}
\includegraphics[height=\figheight\linewidth]{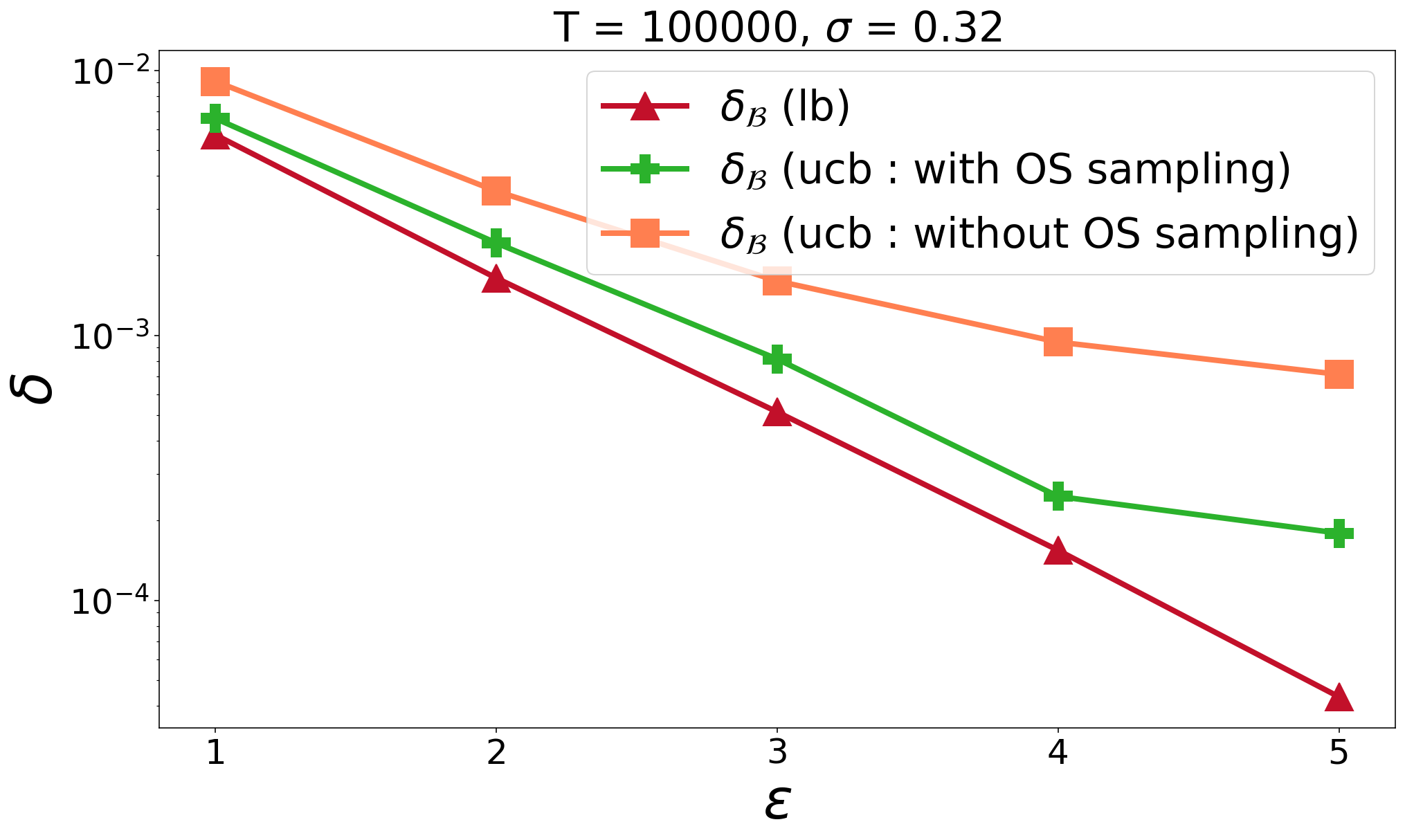} & 
\includegraphics[height=\figheight\linewidth]{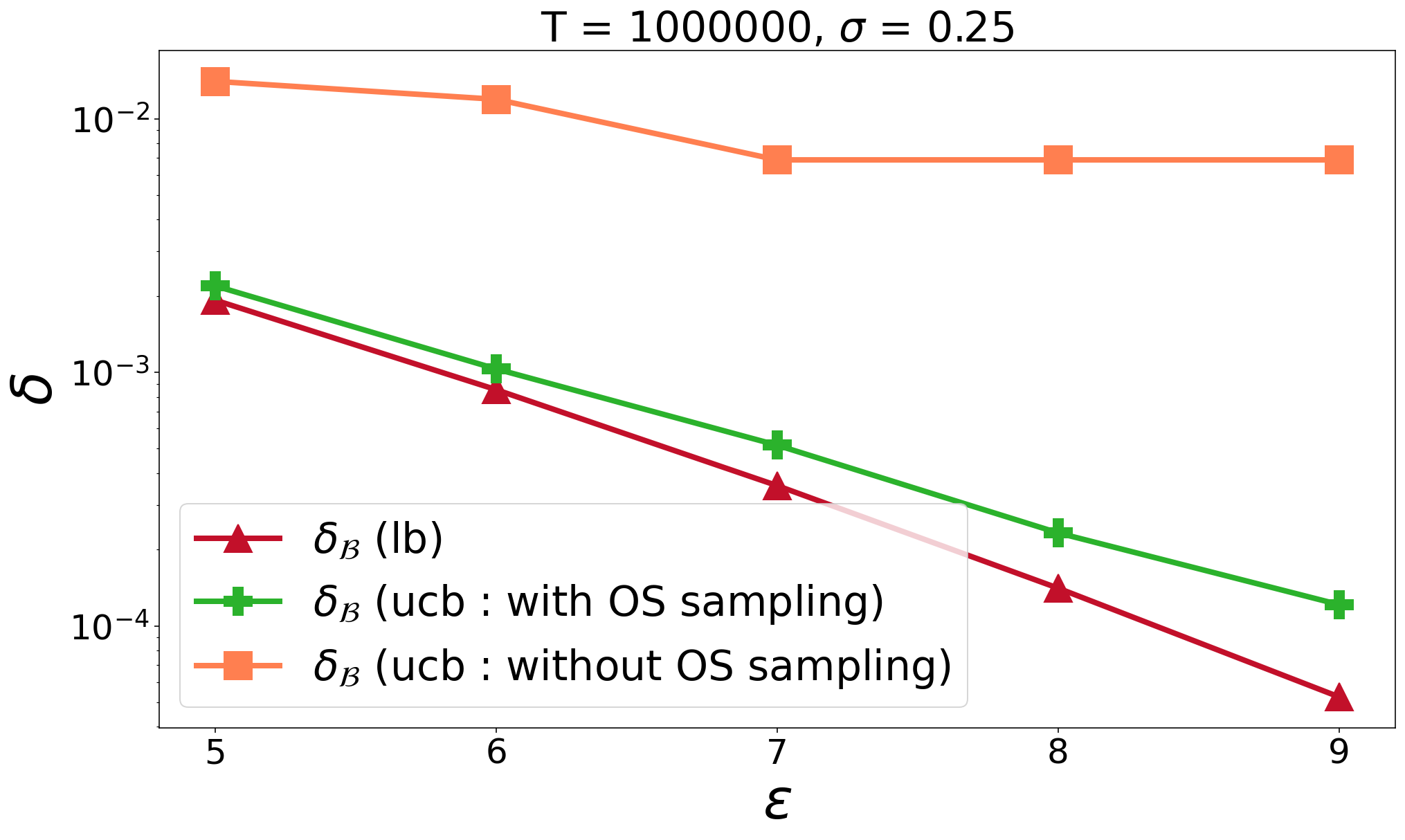}
\end{tabular}
\caption{Upper confidence bounds on $\deltaB(\eps)$ against various values of $\eps$ for two settings of $T$ and $\sigma$, with and without order statistics sampling for roughly the same running time complexity. Since order statistics sampling offers a significant speed up, it affords a larger sample complexity. Additionally, lower bounds on $\deltaB(\eps)$ are included.}
\label{fig:order-stats-advantage}
\end{figure*}

\subsection{Combining Importance and Order Statistics Sampling}\label{subapp:importance-order-combined}

We sketch how the techniques of importance sampling and order statistics sampling can be used together.

\paragraph{\boldmath Estimating $\Deps(\QB ~\|~ \PB)$.} In this case, we wish to estimate $\Ex_{x \sim \QB |_{E_\eps}} \max\{0, 1 - e^{\eps - L_{\QB ~\|~ \PB}(x)}\}$ where $\QB = \cN(0, \sigma^2 I)$ and $E_\eps := \{x :\in \R^T : \max_{t \in [T]} x_t \le C_{\eps}\}$ for $C_\eps := \frac12 - \eps \sigma^2$. We can sample the order statistics $y^{(k_1)}, \ldots, y^{(k_r)}$ for $x_1, \ldots, x_T \sim \QB|_{E_\eps}$ by using a small variant of \Cref{alg:order-sampling} wherein we set $y^{(k_i)} \gets \Phi_{\sigma}^{-1}(\Phi_\sigma(C_\eps) \cdot \prod_{j=1}^i z_j)$, where the term $\Phi_\sigma(C_\eps)$ essentially implements the conditioning on $x_1, \ldots, x_T \le C_\eps$, via \Cref{prop:sample-single-order}. And finally, we use \Cref{alg:importance-mce} where we replace $L_{\QB ~\|~ \PB}(x)$ by an upper bound in terms of the order statistics given as,
\[\textstyle
\log T + \frac{1}{2\sigma^2} - \log\prn{\sum_{i=1}^r (k_i - k_{i-1}) \cdot e^{y^{(k_i)} / \sigma^2}}.
\]

\paragraph{\boldmath Estimating $\Deps(\PB ~\|~ \QB)$.} In this case, we wish to estimate $\Ex_{x \sim P_1 |_{E_\eps}} \max\{0, 1 - e^{\eps - L_{\PB ~\|~ \QB}(x)}\}$ where $P_1 = \cN(e_1, \sigma^2 I)$ and $E_\eps := \{ x : \max\{ x_1 - 1, \max_{t > 1} x_t\} \ge C_\eps \}$ for
$C_{\eps} ~=~ \frac{1}{2} + \sigma^2 \cdot \prn{\eps - \log\prn{1 + \frac{e^{1/\sigma^2} - 1}{T}}}$. Recall that for $x \sim P_1 |_{E_\eps}$, the distribution of $x - e_1$ is the same as $\cN(0, \sigma^2 I_T)|_{\{x\,:\,\max_t x_t \ge C_\eps\}}$.
We follow the first two steps of \Cref{alg:max-of-samples-from-P} and sample $y_* \sim \Beta(T, 1)|_{[\Phi_\sigma(C_\eps), 1]}$, and sample $t_*$ uniformly at random in $[T]$.
There are two cases to handle:
\begin{itemize}
\item If $t_* = 1$, then we set $x_1 = y_* + 1$ and use \Cref{alg:order-sampling} to sample the order statistics $y^{(k_1)}, \ldots, y^{(k_r)}$ with $R = T-1$ and use the following upper bound on $L_{\PB ~\|~ \QB}$:
\[
\log\prn{e^{x_1/\sigma^2} + \sum_{i=1}^r (k_{i+1} - k_i) \cdot e^{y^{(k_i)}/\sigma^2}} - \log T - \frac{1}{2\sigma^2}.
\]
\item If $t_* \ne 1$, we can assume without loss of generality, that $t_* = 2$. In this case, we set $x_2 = y_*$. We sample $x_1 \sim \cN(1, \sigma^2)|_{(-\infty, y_*]}$ using \Cref{fact:prob-int-transform}, and use a small variant of \Cref{alg:order-sampling} to sample the order statistics $y^{(k_1)}, \ldots, y^{(k_r)}$ with $R = T-2$, wherein we set $y^{(k_i)} \gets \Phi_\sigma^{-1}(\Phi_\sigma(y_*) \cdot \prod_{j=1}^i z_j)$ and use the following upper bound on $L_{\PB ~\|~ \QB}(x)$:
\[
\log\prn{e^{x_1/\sigma^2} + e^{x_2/\sigma^2} + \sum_{i=1}^r (k_{i+1} - k_i) \cdot e^{y^{(k_i)}/\sigma^2}} - \log T - \frac{1}{2\sigma^2}.
\]
\end{itemize}

\subsection{\boldmath Privacy Accounting of \texorpdfstring{$\ABLQB$}{ABLQ\_B} for Multiple Epochs}\label{subapp:multiple-epochs}

While the focus in this work was on the case of a {\em single epoch} of training (i.e., with a single pass over the training dataset), the Monte Carlo sampling approach extends to the case of $k$ epochs since the $\deltaB(\eps) := \max\{\Deps(\PB^{\otimes k} ~\|~ \QB^{\otimes k}), \Deps(\QB^{\otimes k} ~\|~ \PB^{\otimes k})\}$. This can be estimated using \Cref{alg:mce-multiple-epochs}, which relies on the simple observation that
\[
L_{P^{\otimes k} ~\|~ Q^{\otimes k}}(x^{(1)}, \ldots, x^{(k)}) ~=~ \sum_{i=1}^k L_{P ~\|~ Q}(x^{(i)})\,.
\]
The order statistics sampling technique can be extended to this case by applying it independently to sample an upper bound on $L_{P ~\|~ Q}(x^{(i,j)})$ for each $j \in [k]$.  Importance sampling is not directly applicable though, and we leave it to future work to construct importance samplers for the multi-epoch case.

\begin{algorithm}[t]
\caption{Monte Carlo Estimation of $\Deps(P^{\otimes k} ~\|~ Q^{\otimes k})$.}
\label{alg:mce-multiple-epochs}
\begin{algorithmic}
\REQUIRE Distributions $P$ and $Q$; sample access to $P$, Number of epochs $k$, Sample size $m$, Error probability $\beta$.
\ENSURE An upper confidence bound on $\Deps(P^{\otimes k} ~\|~ Q^{\otimes k})$.
\STATE Sample $x^{(i,j)} \sim P$ for $i \in [m]$ and $j \in [k]$
\STATE $q \gets \frac1m \sum_{i=1}^m \max\{ 0, 1 - e^{\eps - \sum_{j=1}^k L_{P ~\|~ Q}(x^{(i,j)})}\}$
\STATE $p \gets$ smallest value in $[q, 1]$ such that $\KL(q ~\|~ p) \ge \log(1/\beta) / m$, or $1$ if no such value exists
\RETURN $p$
\end{algorithmic}
\end{algorithm}

\section{Training Details} \label{app:training}
We use a neural network with five layers as the model, with around 85M parameters for the Criteo pCTR dataset and 57M parameters for the Criteo Search Conversion Logs dataset. The first layer consists of feature transforms. Categorical features are mapped into dense feature vectors using an embedding layer, with embedding dimension of $48$ each. For the Criteo Search dataset, we treat all features as categorical features, whereas for the Criteo pCTR dataset, we apply a log transform for the integer features. We concatenate all the features together, and feed them into three fully connected layers with $598$ hidden units each and a ReLU activation function. The last layer is a fully connected layer that gives a scalar (\texttt{logit}) prediction.

We use the Adam optimizer with a base learning rate in $\{0.0001, 0.0005, 0.001, 0.005, 0.01\}$, which is scaled with a cosine decay. We use batch sizes that are powers of $2$ between $1024$ and $262\,144$, and we tune the norm bound $C\in\{1, 5, 10, 15, 30, 50, 100, 500, 1000\}$. We run the training using NVIDIA Tesla P100 GPUs, where each run takes up to an hour on a single GPU.

Since our implementation of $\DPSGD$ in JAX works with fixed batch sizes, for each set of parameters we pick a maximum batch size $B$ and truncate the batches to have size at most $B$, and batches with size smaller than $B$ are padded with dummy examples with zero weight.
For Poisson subsampling and \BnB sampling, the batch sizes are (marginally) distributed as the binomial distribution $\mathsf{Bin}(n, b/n)$.
\citet[Proposition 3.2, Theorem 3.3]{chua24scalable} showed that for a given expected batch size $b$, the total number of examples $n$, the number of training steps $T$, and a maximum batch size of $B$, $\ABLQP$ satisfies $(\eps, \deltaP(\eps) + \delta')$-DP for $\delta' = (1 + e^\eps) T \cdot \Pr_{r \sim \mathrm{Bin}(n, b/n)}[r > B]$. The same argument also applies in the case of \BnB sampling.
In our experiments, we choose $B$ such that this quantity is at most
$\delta' \le 10^{-10}$ even at $\eps = 10$, so the change in the $\delta$ values is negligible relative to the values of $\deltaP(\eps)$ and $\deltaB(\eps)$ we consider.
In particular, we use maximum batch sizes in $\{1328, 2469, 4681, 9007, 17\,520, 34\,355, 67\,754, 134\,172, 266\,475\}$ for the Criteo pCTR dataset and $\{1320, 2458, 4665, 8984, 17\,488, 34\,309, 67\,687, 134\,071, 266\,317\}$ for the Criteo Search dataset, corresponding to the expected batch sizes of $\{1024, 2048, 4096, 8192, 16\,384, 32\,768, 65\,536, 131\,072, 262\,144\}$.

For the privacy accounting in \Cref{fig:privacy}, we use order statistics sampling with the following set of order indices:
\begin{itemize}
\item For Criteo pCTR dataset, there are a total of $37\,000\,000$ examples in the training set.
	\begin{itemize}[label=$\triangleright$]
	\item For expected batch size $1024$, there are total of $T = 36\,133$ steps. We use the order statistics of $(1, 2, \ldots, 500, 510, 520, \ldots, 1000, 1100, 1200, \ldots, 19\,900)$, which involves a total of $739$ orders, which is about $2\%$ of the number of steps.
	\item For expected batch size $8192$, there are a total of $T = 4517$ steps. We use the order statistics of $(1, 2, \ldots, 500, 510, 520, \ldots, 1000, 1050, 1100, \ldots, 2950)$, which involves a total of $589$ orders, which is about $13\%$ of the number of steps.
	\end{itemize}
\item For Criteo Sponsored Search Conversion Log dataset, there are a total of $12\,796\,151$ examples in the training set.
	\begin{itemize}[label=$\triangleright$]
	\item For expected batch size $1024$, there are total of $T = 12\,497$ steps. We use the order statistics of $(1, 2, \ldots, 500, 510, 520, \ldots, 1000, 1100, 1200, \ldots, 6900)$, which involves a total of $609$ orders, which is about $4.8\%$ of the number of steps.
	\item For expected batch size $8192$, there are a total of $T = 1563$ steps. We use the order statistics of $(1, 2, \ldots, 500, 510, 520, \ldots, 990)$, which involves a total of $589$ orders, which is about $35\%$ of the number of steps.
	\end{itemize}
\end{itemize}
For efficiency, instead of applying Monte Carlo estimation using independent samples for each $\eps$, we instead generate $5 \cdot 10^8$ samples of upper bounds on $L_{\PB ~\|~ \QB}(x)$ (resp. $L_{\QB ~\|~ \PB}(x)$) using the order statistics sampling~(\Cref{alg:order-sampling}), and subsequently use them to estimate $\Deps(\PB ~\|~ \QB)$ (resp. $\Deps(\QB ~\|~ \PB)$). For this reason, we do not use importance sampling here since that depends on each $\eps$. The computation was performed in parallel on a cluster of 60 CPU machines.

\section{\boldmath Incomparability of Dominating Pairs for \texorpdfstring{$\ABLQB$}{ABLQ\_B} and \texorpdfstring{$\ABLQP$}{ABLQ\_P}}\label{app:bnb-vs-poisson-incomparable}

We elaborate on \Cref{rem:bnb-poisson-incomparable} showing that $\ABLQB$ and $\ABLQP$ have incomparable privacy guarantees.

\begin{theorem}\label{thm:bnb-vs-poisson-2}
For all $\sigma > 0$ and $T > 1$, there exists $\eps_0, \eps_1 \in \R$ such that
\begin{enumerate}[label=(\alph*),topsep=0pt,itemsep=0pt]
\item $D_{e^{\eps_0}}(\PB ~\|~ \QB) > D_{e^{\eps_0}}(\PP ~\|~ \QP)$, and
\item $D_{e^{\eps_1}}(\PB ~\|~ \QB) < D_{e^{\eps_1}}(\PP ~\|~ \QP)$.
\end{enumerate}
\end{theorem}

We use the following lemma regarding KL divergence, defined for probability distributions $P$, $Q$ over the space $\Omega$ as $\KL(P ~\|~ Q) := \int_{\Omega} \log \frac{dP}{dQ}(\omega) \cdot dP(\omega)$.
\begin{lemma}\label{lem:kl-inequality}
Let $P$ be a joint distribution over $\Omega := \Omega_1 \times \cdots \times \Omega_n$, and let $Q = Q_1 \otimes \cdots \otimes Q_T$ be a product distribution over $\Omega$. Then, for $P_1, \ldots, P_T$ being the marginal distributions of $P$ over $\Omega_1, \ldots, \Omega_T$ respectively, it holds that
$$\KL(P ~\|~ Q) \ge \KL(P_1 \otimes \dots \otimes P_T ~\|~ Q)\,.$$
Moreover, equality holds if and only if $P$ is a product distribution.
\end{lemma}

\begin{fact}[Post-processing inequality for KL-divergence]\label{fact:kl-post-processing}
For distributions $P, Q$ over $\Omega$, and distributions $A, B$ over $\Gamma$, if there exists $f : \Omega \to \Gamma$ such that $f(P) = A$ and $f(Q) = B$, then $\KL(P ~\|~ Q) \ge \KL(A ~\|~ B)$.
\end{fact}

\begin{lemma}[{Converse to \Cref{lem:post-process-and-domination}; \cite[Theorem 2.5]{kairouz15composition}}]\label{lem:converse-post-process-and-domination}
For distributions $P, Q$ over $\Omega$, and distributions $A, B$ over $\Gamma$, if $(P, Q) \dominates (A, B)$ then there exists $f : \Omega \to \Gamma$ such that simultaneously $f(P) = A$ and $f(Q) = B$.
\end{lemma}

\begin{proof}[Proof of \Cref{thm:bnb-vs-poisson-2}]
Part (a) follows from the observation that $\QP = \QB$ and $\PP$ can be obtained as simply the product of the marginal distributions of $\PB$. Thus, applying \Cref{lem:kl-inequality}, we get that $\KL(\PB ~\|~ \QB) > \KL(\PP ~\|~ \QP)$, with strict inequality because $\PB$ is not a product distribution. Thus, by the contrapositive of \Cref{fact:kl-post-processing}, we get that there does not exist a post-processing that simultaneously maps $\PP$ to $\PB$ and $\QP$ to $\QB$. Finally, by \Cref{lem:post-process-and-domination}, we conclude that $(\PP, \QP) \not\dominates (\PB, \QB)$ or in other words, there exists an $\eps_0 \in \R$ such that $D_{e^{\eps_0}}(\PB ~\|~ \QB) > D_{e^{\eps_0}}(\PP ~\|~ \QP)$.

Part (b) follows immediately from \Cref{thm:bnb-vs-poisson}.
\end{proof}

While \Cref{thm:bnb-vs-poisson} gives us that there exists an $\eps \ge 0$ such that $\deltaB(\eps) < \deltaP(\eps)$ for all $\sigma > 0$ and $T > 1$ (in fact, this holds for sufficiently large $\eps$), interestingly \Cref{thm:bnb-vs-poisson-2} does {\em  not} imply that there exists $\eps \ge 0$ such that $\deltaB(\eps) > \deltaP(\eps)$, since $\deltaB(\eps)$ corresponds to $\max\{\Deps(\PB ~\|~ \QB), \Deps(\QB ~\|~ \PB)\}$. Whether there always exists such an $\eps \ge 0$ for all $\sigma > 0$ and $T > 1$ is left open for future investigation.

\vfill %

\end{document}